\documentclass[10pt,twocolumn,twoside]{IEEEtran} 
\IEEEoverridecommandlockouts          

\usepackage{graphicx}
\usepackage{subfigure}
\usepackage{caption}
\usepackage{amsthm} % For proof environment
\usepackage{amsmath}
\usepackage{amssymb}
\usepackage{accents}
\usepackage{statex}
\usepackage{cases}
\usepackage{mathrsfs}
\usepackage{tikz,pgfplots}
\pgfplotsset{compat=newest}
\usepackage{xcolor}
\usepackage{caption}
\usepackage{setspace}
\usepackage{booktabs}
\usepackage{array}
\usepackage{cite}

\usepackage{enumitem}
\usepackage{algorithm}
\usepackage{algorithmic} % <===========================================
\usepackage{multirow}
\usetikzlibrary{patterns}
\usetikzlibrary{graphs}
\usetikzlibrary{graphs.standard}
\newcolumntype{C}[1]{>{\centering\arraybackslash}p{#1}}
\usepackage{eqparbox}

\usetikzlibrary{positioning,fit,calc,arrows.meta,backgrounds,shapes.misc}

\definecolor{sysblue}{RGB}{54,114,201}   % 边框更深
\definecolor{sysfill}{RGB}{226,236,252}
\definecolor{lossgray}{RGB}{225,225,225}
\definecolor{datablue}{RGB}{173,200,255}
\definecolor{omegacol}{RGB}{255,165,100}

\tikzset{
	>={Stealth[length=2.6mm]},
	smallbox/.style={draw, rounded corners=2pt, fill=lossgray, minimum width=2.9cm, minimum height=1.05cm, line width=0.8pt},
	datablock/.style={draw, fill=datablue, minimum width=1.25cm, minimum height=0.7cm, line width=0.8pt},
	omega/.style={draw=omegacol, fill=white, circle, minimum size=7.8mm, inner sep=0pt,
		line width=1.0pt, double, double distance=1.1pt},
	plus/.style={draw, circle, inner sep=1.5pt, minimum size=9mm, line width=1.0pt},
	sysblock/.style={draw=sysblue, fill=sysfill, minimum width=5.2cm, minimum height=1.1cm,
		line width=1.2pt, rounded corners=2pt},
	zoompad/.style={draw=none, fill=datablue!45, rounded corners=3pt},
	dashedframe/.style={draw, rounded corners=3pt, dashed, line width=0.8pt, inner sep=7pt}
}

\usetikzlibrary{automata}

\newtheorem{proposition}{Proposition}

\newtheorem{lemma}{Lemma}

\newtheorem{remark}{Remark}
\newtheorem{assumption}{Assumption}

\newlength\figureheight
\newlength\figurewidth

\DeclareMathOperator*{\argmin}{argmin}

\DeclareMathOperator*{\mD}{\mathcal{D}}

\newcommand{\revise}[1]{\textcolor{black}{#1}}

\title{\bf Meta-Learning for Rapid Adaptation in Reference Tracking of Uncertain Nonlinear Systems}

\author{Jiaqi Yan, Ankush Chakrabarty, Niklas Schmid, John Lygeros, and Alisa Rupenyan
	\thanks{
		J. Yan is with the School of Automation Science and Electrical Engineering, Beihang University, Beijing 100191, P. R. China (jqyan@buaa.edu.cn).
		
		A. Chakrabarty is with the Mitsubishi Electric Research Laboratories, Cambridge, MA 02139, USA (achakrabarty@ieee.org).
		
		N. Schmid and J. Lygeros are with the Automatic Control Laboratory, ETH Zurich, Switzerland (nikschmid@ethz.ch, jlygeros@ethz.ch).
		
		A. Rupenyan is with the ZHAW Centre for Artificial Intelligence, Zurich University of Applied Sciences, Switzerland  (alisa.rupenyan@zhaw.ch).
		This work is supported by the Swiss National Science Foundation through NCCR Automation under Grant agreement 51NF40\_180545.
	}
}

\begin{document}
	\maketitle
	
	\begin{abstract}
		In this paper, we address the problem of reference tracking for uncertain nonlinear systems. Since collecting data from the target system (i.e., the system of interest) is often challenging, our objective is to design optimal controllers using limited target system data. Meta-learning provides a promising paradigm by leveraging offline data from source systems (systems sharing structural similarities with the target system) to accelerate training and enhance control performance. Motivated by this idea, we propose a meta-learning-based control framework that tailors the implicit model-agnostic meta-learning (iMAML) algorithm to the control setting. The framework operates in two phases: an (offline) meta-training phase, where an aggregated representation is learned from source data to capture the shared system dynamics among similar systems, and an (online) meta-adaptation phase, where this representation is fine-tuned on the target system using only a few data samples and limited adaptation steps. We formulate this framework as a bi-level optimization problem and provide an efficient solution with reduced storage complexity and few approximations. The proposed framework is general, allowing various learning algorithms to be integrated. To demonstrate this flexibility, we propose two specific learning algorithms that can be incorporated into our framework based on a neural state-space model and a deep Q-network, respectively. The primary distinction between these approaches is whether explicit system identification is required. Numerical simulations and hardware experiments demonstrate that the proposed methods enhance control performance and consistently outperform baseline approaches.
	\end{abstract}
	
	\def\abstractname{Note to Practitioners}
	\begin{abstract}
		This work is motivated by the fact that collecting sufficient data from physical systems for control design is often costly and time-consuming. To address this issue, we propose a meta-learning-based control framework that enhances adaptation speed and control performance when only limited target system data are available. The proposed approach leverages historical data from similar systems to pre-train a model offline, which can then be efficiently adapted to new systems online with minimal data and computation. Although we demonstrate the framework using two specific algorithms, it is general and readily extendable to various learning-based control methods. Practitioners can apply this framework to improve control performance and data efficiency in real-world applications such as autonomous driving, robotic manipulation, and process control, where conventional data-driven methods typically require extensive online training. Experimental results show that our methods outperform existing meta-learning and data-driven control baselines in both simulation and hardware experiments. Moreover, the results highlight the effectiveness of the proposed approach in bridging the sim-to-real gap. Future work will focus on extending this framework to multi-agent systems and enabling real-time adaptation in highly dynamic environments.
	\end{abstract}
	

	%%%%%%%%%%%%%%%%%%%%%%%%%%%%%%%%%%%%%%%%%%%%%%%%%%%%%%%%%%%%%%%%%%%%%%
	\section{Introduction}

	Control systems that can easily adapt to new scenarios and changing dynamics are crucial for enabling flexible automation in various application fields, in particular manufacturing and robotics. Learning-based approaches to adaptation, such as reinforcement learning, adaptive dynamic programming, and stochastic optimization-based control \cite{dierks2012online,modares2017optimal,pepyne2000optimal} offer a promising approach to this problem. %change citations include review
	Meta-learning \cite{vilalta2002perspective} can complement these approaches when there is limited access to data from the system to be controlled (\textit{target system}). It combines data from \textit{source systems} that share structural similarities with the target system. Thus, data from numerical models, digital twins, or similar physical systems can be used to train the meta-learning model before deployment on the target system. The model is then adapted with limited evaluations from the target system. Therefore, meta-learning comprises two phases: a meta-training phase that pretrains a model to capture similarities among source systems, and a meta-adaptation phase that fine-tunes the pretrained model on the target system.  \revise{When the target and the source systems are similar, it is also possible to use transfer learning to pass knowledge from different source systems. %{Usually, transfer learning requires more data than meta-learning to train a separate model for the target system. 
    %The distinction between transfer learning and meta-learning is discussed in \cite{dumoulin2021comparing}. 
Transfer learning typically requires sufficient target system data to fine-tune a pre-trained model, and the transferred knowledge often comes from a single source task. On the other hand, multi-task learning assumes all tasks are available during training and optimizes for joint performance rather than rapid adaptation \cite{Cui2025,Pei2023}. In contrast, meta-learning  optimizes for rapid adaptation to new tasks with minimal data, which is reflected in the bi-level optimization structure of meta-learning: the inner loop simulates adaptation on individual tasks, while the outer loop optimizes for fast adaptation capability itself. This structure is absent in standard transfer or multi-task learning formulations.}
	
	The meta-learning approach is particularly suitable for nonlinear systems with uncertainties, as it eliminates the need for a complete re-design of the controller. Most meta-learning-based control algorithms, e.g., \cite{toso2024meta,chakrabarty2023meta,zhan2022calibrating}, rely on the Model Agnostic Meta Learning (MAML) \cite{finn2017model} to enable rapid adaptation to new tasks. MAML involves solving a bi-level optimization problem in the meta-training stage. The inner problem calculates and propagates derivatives of the training loss function along the full optimization path (the length of which is the same as the number of inner-loop updates/steps). The outer problem then accumulates information from these inner-loop updates and computes an outer-loop gradient direction. By repeating these two steps, the MAML learner is expected to asymptotically converge to a pre-trained controller from which rapid adaptation is possible on the target system. Choosing a large number of inner-loop steps, while leading to better outer-loop directions, incurs high memory complexity, and is not practical for a large number of inner-loop adaptation steps. Existing approaches propose various approximations to obtain efficient outer loop estimates, or cast the problem in mathematical frameworks that reduce training time. For example, in \cite{finn2017model}, a first-order approximation is used for specific rectified linear unit (ReLU) activation functions. However, such approximations sometimes oversimplify the computations and degrade learning performance. Alternatively, the inner loop is solved as an evolution of the ordinary differential equation (ODE) with gradients obtained through the adjoint ODE in \cite{pmlr-v206-li23c_adjoint}. 
	
	The aforementioned methods, while improving computational efficiency, often rely on gradient approximations. To overcome this limitation, the implicit MAML (iMAML) algorithm leverages the implicit function theorem to achieve efficient outer-loop gradient computation with fewer approximations \cite{rajeswaran2019meta}. This approach solves the inner-loop optimization over many iterations without incurring excessive memory overhead, as it is agnostic to the optimization path and depends only on the final gradient direction. Moreover, minimizing such approximations can enhance adaptation performance. The problem of reference tracking in uncertain nonlinear systems making use of iMAML was discussed in \cite{yan2024mpc}. The gradients were computed exactly during training without relying on the entire optimization path, and achieved good predictive performance for model predictive control (MPC) in different (related) systems.
	
We build upon the approach from \cite{yan2024mpc} and propose here a general meta-learning-based control framework. \revise{By ``general'', we mean that the framework can accommodate diverse learning-based control approaches without being restricted to a specific algorithm family. Specifically, the framework supports any control method that can be formulated with differentiable loss functions, thereby unifying
        \begin{itemize}
            \item \textit{Indirect data-driven methods}, where a model is first identified and then used for control design (as in \cite{yan2024mpc}), and
\item \textit{Direct data-driven methods}, where the control policy is optimized directly from data.
        \end{itemize}
The comparison between these methods is illustrated in Fig.~\ref{fig:direct_indirect}. In contrast,~\cite{yan2024mpc} focuses on a particular indirect implementation based on neural state-space models (NSSMs) for system identification. As such, it can be viewed as a special case of the framework proposed here.}

\begin{figure}
    \centering
    \includegraphics[width=0.85\linewidth]{   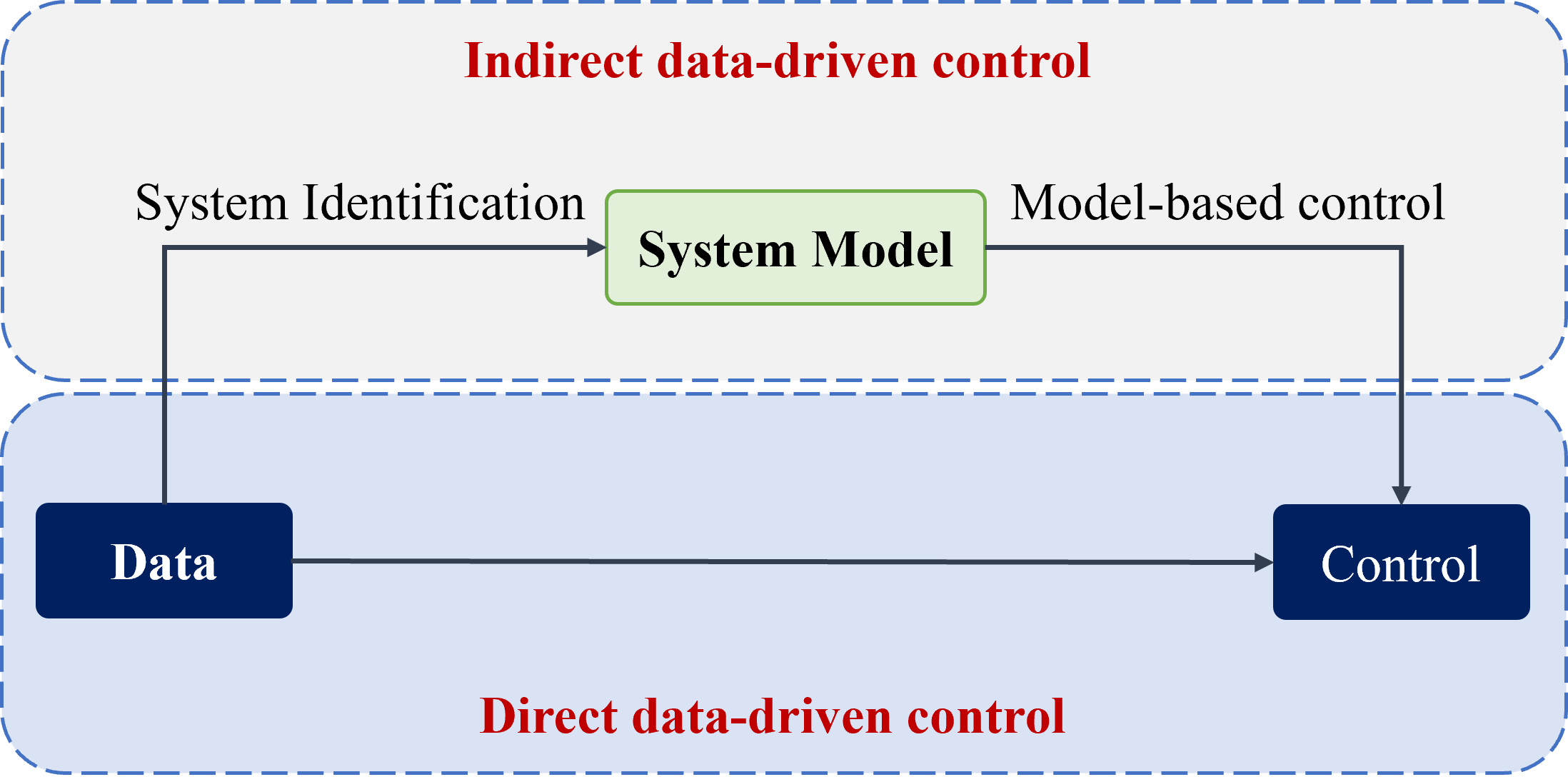}
    \caption{\revise{Direct vs. indirect data-driven control methods.}}
    \label{fig:direct_indirect}
\end{figure}

The main contributions of this paper are: %rewrite this
	\begin{enumerate}

        \item \revise{While meta-learning has been shown to be effective for both model adaptation \cite{lin2020system,chakrabarty2025meta,chakrabarty2023meta} and control policy adaptation \cite{harrison2018control,ghadirzadeh2021bayesian,chen2024meta}, existing implementations are typically applied on a case-by-case basis. In contrast, our framework is general and flexible, capable of incorporating various learning algorithms. Unlike existing implementations that are algorithm-specific, our framework provides a unified mathematical structure based on bi-level optimization that remains invariant across different learning algorithms---only the specific form of the loss function changes to reflect the particular control approach employed. }

		\item To demonstrate the flexibility of our framework, we propose two representative algorithms: an \textit{indirect data-driven control} method based on neural state-space models (NSSM), and a \textit{direct data-driven control} method based on a Deep Q-Network (DQN). The key difference lies in whether explicit system identification is required, which is reflected in the corresponding loss function design. We show that both methods enable efficient adaptation to unseen target tasks while maintaining strong tracking performance under system uncertainties.

        \item \revise{Under standard assumptions, we establish convergence guarantees for the proposed framework and provide stabilization analysis for the resulting control algorithms. This moves beyond simply applying learning algorithms to control problems, thereby highlighting the control-theoretic foundations of the proposed approach.}
		
	\end{enumerate}
	
	Numerical simulations demonstrate that our iMAML-based method outperforms multiple baselines, including both MAML-trained models and models trained purely on target system data. \revise{Moreover, hardware experiments on the ball-on-a-plate system are particularly designed to evaluate the direct meta-learning approach. Experimental results verify that by pre-training the controller in simulation, the proposed algorithm enables faster adaptation to the real system while requiring significantly less real-world data compared to training solely on physical system data. This approach effectively reduces the sim-to-real gap in control system implementation.}
	
	The remainder of this paper is organized as follows: Section~\ref{sec:problem} formulates the problem and introduces the datasets. Section~\ref{sec:meta_learning} presents the general meta-learning framework and its solution. Section~\ref{sec:indirect} develops the indirect and direct data-driven algorithms, respectively, that can be integrated into the proposed framework. Finally, numerical simulations and hardware experiments are given in Section~\ref{sec:sim} to validate the proposed algorithms, and Section~\ref{sec:conclude} concludes the paper.

	\section{Problem Formulation}\label{sec:problem}
	We consider a family of parameterized discrete-time nonlinear systems of the form 
	\begin{equation}\label{eqn:sys}
		\begin{split}
			x_{t+1} &= f(x_t, u_{t+1}, \theta_f), \\
			y_t &= g(x_t, \theta_g),
		\end{split}
	\end{equation}
	where $x \in \mathbb{R}^n$ and \revise{$y \in\mathcal{Y} \subseteq \mathbb{R}^m$} represent the system state and \revise{(constrained) output}, and $u \in \mathcal{U} \subseteq \mathbb{R}^p$ is the system input, where \revise{$\mathcal{U}$ and $\mathcal{Y}$ are compact, convex sets representing the input and output constraints.} Both $f$ and $g$ are nonlinear functions, and $\theta := [\theta_f; \theta_g] \in \mathbb{R}^w$ denotes a vector of unknown parameters. 
	
	The target system is described by \eqref{eqn:sys} with parameters $\theta^0$. The true value of $\theta^0$ is unknown but is drawn from a distribution  
	$
	\theta^0 \sim \Theta,
	$
	where $\Theta$ is often informed by domain knowledge. We assume that we can obtain samples for $\theta$, even if the distribution and support set $\Theta$ are unknown.
	
	Let $U(\theta^0, T^0)$ and $Y(\theta^0, T^0)$ represent the input and output trajectories generated from the target system over a horizon $T^0$. We construct the target dataset as
	\begin{equation}
		\mathcal{D}_{\text{target}} := \{U(\theta^0, T^0), Y(\theta^0, T^0)\}.
	\end{equation}
	Despite the uncertain dynamics, our objective is to design a controller to asymptotically track some reference signal $\bar{y}$ such that
	\begin{equation}\label{eqn:objective}
		\lim_{t\to\infty} y_t = \bar{y}_{t}.
	\end{equation}

	In many applications, collecting data from the target system is expensive. Therefore, $\mathcal{D}_{\text{target}}$ is typically of limited size. Training with such limited data can lead to poor control performance. To overcome this challenge, we propose a meta-learning-based control framework. The key idea is to learn an \textit{aggregated representation}—typically offline—using data from multiple source systems that share structural similarities with the target system. \revise{The aggregated representation captures structural similarities among the source systems and serves as an informative initialization for rapid adaptation to a target system with limited data. From an engineering perspective, this design is motivated by practical scenarios in automation and robotics where collecting extensive data from the physical target system is costly or potentially unsafe, while simulation models, digital twins, or similar systems are readily available offline. In such cases, leveraging source-system data to learn a shared representation significantly improves data efficiency.} 
    
     \revise{Also note that the effectiveness of the aggregated representation depends on the similarity between source and target systems. If the target dynamics differ significantly from the source systems, the adaptation performance may deteriorate. As such, representative source-system selection is important for successful deployment.} Accordingly, we assume that each source system follows the dynamics in \eqref{eqn:sys}, where the $k$-th system is parameterized by $\theta_k$ with $\theta_k \sim \Theta$. It is important to note that the each value of $\theta_k$ need not be known. 

Assume that we have access to a dataset consisting of input–output trajectories collected from $N_s$ different source systems. We denote the resulting source dataset as 
	\begin{equation}
		\begin{split}
			\mathcal{D}_{\text{source}}&= \{(U(\theta^k, T^k), Y(\theta^k, T^k))\}_{k=1}^{N_s}\\&:= \{\mathcal{D}^k\}_{k=1}^{N_s}, 
		\end{split}
	\end{equation}
	where $T^k$ is the length of trajectory collected from the $k$-th system. 
    \revise{Note that when the source dataset contains only a single system ($N_s=1$), our framework reduces to standard transfer learning. Meta-learning generalizes this by learning from multiple source systems and optimizing  for adaptation efficiency through the bi-level structure.}

	%%%%%%%%%%%%%%%%%%%%%%%%%%%%%%%%%%%%%%%%%%%%%%%%%%%%%%%%%%%%%%%%%%%%
	\section{A Meta-Learning Framework}\label{sec:meta_learning}
	The proposed meta-learning framework is illustrated in Fig.~\ref{fig:phase}. The framework comprises two phases: an \textit{offline} meta-training phase and an \textit{online} meta-adaptation phase. During the meta-training phase, an aggregated representation is learned using source data to capture similarities among the source systems. In the meta-adaptation phase, the representation is quickly fine-tuned for the target system with limited target data and few online adaptation steps. In this section, we do not assume any specific loss functions to achieve the control objective. \revise{The framework is general in the sense that it accommodates any learning-based control approach that can be formulated through differentiable loss functions $\widehat
\ell(\cdot)$. This includes both indirect methods (e.g., model-based approaches such as neural state-space models combined with MPC) and direct methods (e.g., reinforcement learning policies such as DQN, actor-critic methods, or policy gradient algorithms). The only requirement is that the inner and outer loss functions be differentiable with respect to their parameters, allowing gradient-based optimization.}
    
In the squeal, we show how to effectively optimize these parameters to obtain the aggregated representation. Note that only source data is used in this stage. Later in Section~\ref{sec:indirect}, we will design specific loss functions and incorporate target data to achieve the tracking objective~\eqref{eqn:objective}.
	\begin{figure}[!tb]
		\centering
		\includegraphics[width=0.9\linewidth]{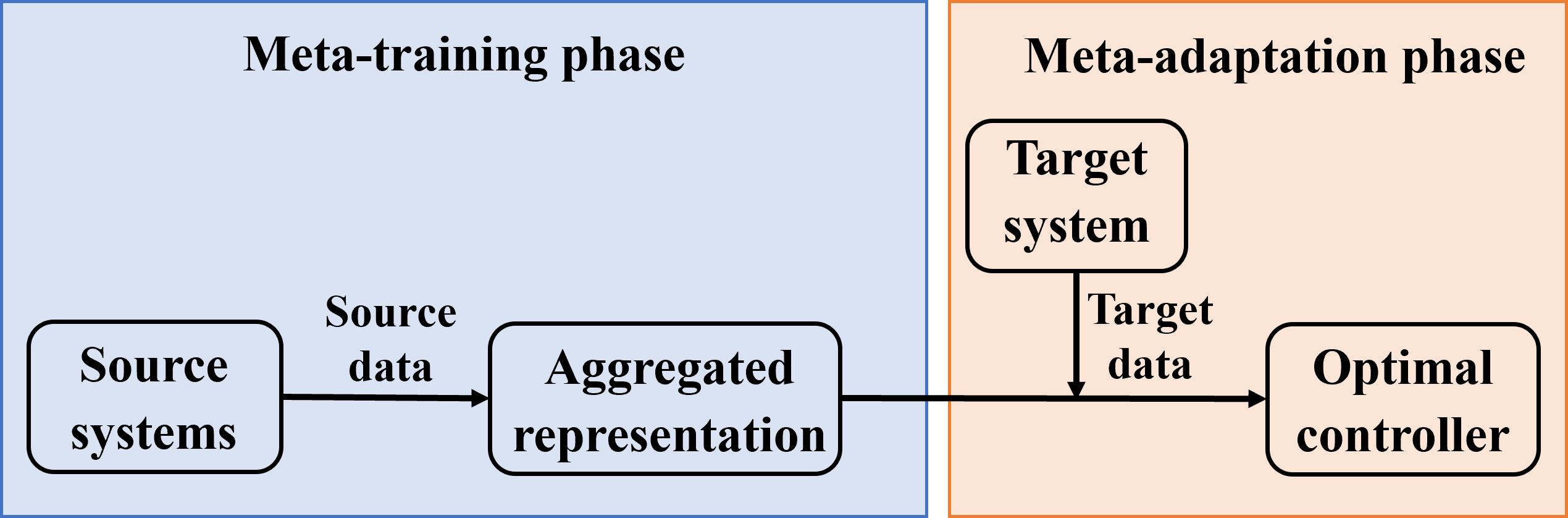}
		\caption{The two phases in the meta-learning framework.}\label{fig:phase}
	\end{figure}

	\subsection{Bi-level optimization problem in meta-training}\label{sec:bilevel}
	
	\begin{figure}[!tb]
		\centering
		\includegraphics[width=0.8\linewidth]{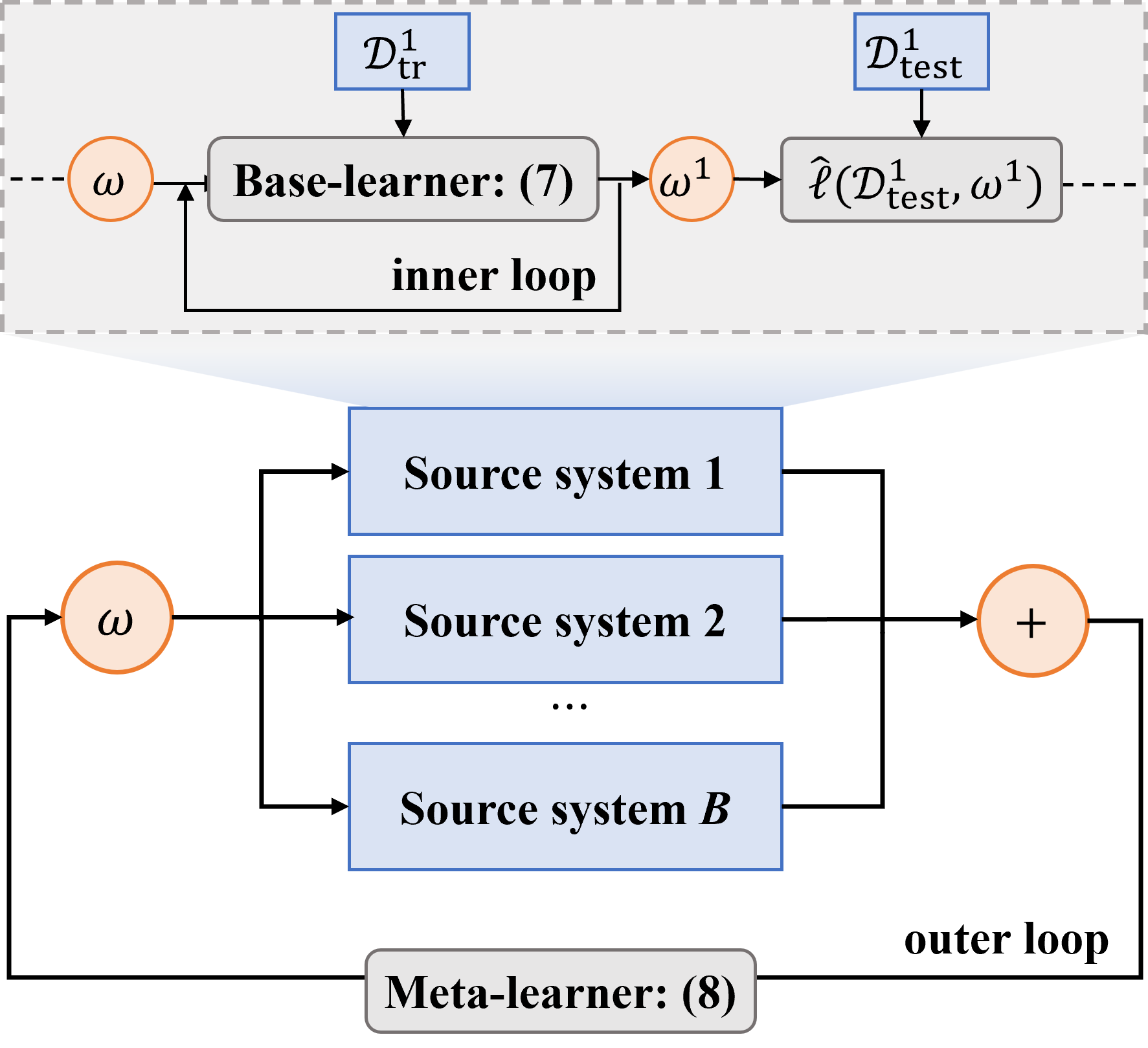}
		\caption{Information flow in the meta-training phase.}\label{fig:diag}
	\end{figure}

	%\begin{figure}
	%	\hspace{10pt}
	%    %\resizebox{0.7\linewidth}{!}{\input{information_flow.tikz}}
	%	\resizebox{0.5\textwidth}{!}{\input{information_flow.tikz}}
	%	\caption{\revise{Information flow in the meta-training phase.}}
	%	\label{fig:diag}
	%\end{figure}

	As shown in Fig.~\ref{fig:diag}, the meta-training phase has two loops. In each outer-loop iteration, a batch of trajectories 
	\( \{\mathcal{D}^b\}_{b=1}^B \) is sampled from the source dataset 
	\( \mathcal{D}_{\text{source}} \). A \textit{meta-learner} updates the aggregated representation $\omega$ by evaluating performance across multiple source systems. In the inner loop, a \textit{base-learner} adapts the aggregated representation \( \omega \), generating system-specific parameters 
	\( \omega^b \) using each source data \( \mathcal{D}^b \). Next, we describe the procedure in detail.
	
	\textit{Inner problem:} The base-learner in the inner problem focuses on adapting the aggregated representation $\omega$ to each individual source system. We partition the dataset associated with the $b$-th source system, denoted by $\mathcal{D}^b$, into a training set $\mathcal{D}^b_{\text{tr}}$ and a testing set $\mathcal{D}^b_{\text{test}}$. The training set is used to fine-tune the aggregated parameters $\omega$ into system-specific parameters $\omega^b$ on each source system by solving the following optimization problem:
	\begin{equation}\label{eqn:inner_imaml}
		\text{Inner Problem: }\tilde\omega^b(\omega):=\argmin_{\psi}  \widehat{\ell}(\mathcal{D}^b_{\text{tr}}; \psi)+\frac{\gamma}{2}||\psi-\omega||^2.
	\end{equation}
	Note the dependence of $\tilde\omega^b(\omega)$ on the aggregated representation $\omega$. In this formulation, $\widehat{\ell}(\cdot)$ denotes the inner-problem loss function that quantifies the adaptation performance for the $b$-th source system. The regularization term $||\psi - \omega||^2$ encourages the adapted parameters to remain close to the aggregated representation, with the regularization strength governed by the parameter $\gamma > 0$. As will be shown in Section~\ref{sec:appro}, this regularization plays a critical role by enabling a closed-form expression for the outer-problem gradient via the Implicit Function Theorem.
	
	\textit{Outer problem:}
	The testing datasets from multiple source systems, denoted by $\{\mathcal{D}^b_{\text{test}}\}_{b=1}^B$, are utilized in the outer problem to update the aggregated representation $\omega$. This update is performed by the meta-learner that assesses the adapted performance across all source systems, as defined by the following objective:
	\begin{equation}\label{eqn:outer_prob}
		\text{Outer Problem: } \min_\omega  L(\omega):=\sum_{b=1}^B \widehat{\ell}(\mathcal{D}^b_{\text{test}}; \tilde\omega^b(\omega)).
	\end{equation}
	The meta-learner aims to learn a set of parameters $\omega$ that can produce good task-specific parameters in different systems.
	
	Combining \eqref{eqn:inner_imaml} and \eqref{eqn:outer_prob}, the meta-training phase is formulated as a bi-level optimization problem. As shown in Fig.~\ref{fig:diag}, these two problems are optimized alternately until convergence is achieved.
	
	\subsection{Solution to the bi-level optimization problem}\label{sec:appro}
	
	The inner optimization problem \eqref{eqn:inner_imaml} can be solved using iterative methods such as gradient descent:
	\begin{equation}\label{eqn:inner_gd}
		\omega^b(\omega)\gets \omega^b(\omega) -\beta_{\text {in }} \big(\nabla_\psi\widehat{\ell}(\mathcal{D}^b_{\text{tr}}; \psi)|_{\psi=\omega^b(\omega)}+\gamma(\omega^b(\omega)-\omega)\big),
	\end{equation}
	where $\beta_{\text{in}}$ denotes the inner-loop learning rate. The same is true for \eqref{eqn:outer_prob}:
	\begin{equation}\label{eqn:outer}
		\omega\gets \omega-\beta_{\text {out }} \sum_{b=1}^B\nabla_\omega  \widehat{\ell}(\mathcal{D}^b_{\text{test}}; \omega^b(\omega)),
	\end{equation}
	where $\beta_{\text {out }}$ is the outer-loop learning rate. \revise{Here we use $\tilde{(\cdot)}$ to denote the exact optimizer, while the notation without $\tilde{(\cdot)}$ represents the solution obtained through iterative approximations.} Note that the dependence of $\omega^b$ on $\omega$ complicates the gradient computation in \eqref{eqn:outer}. By applying the chain rule, we expand each gradient term as
	\begin{equation}\label{eqn:meta_gradient}
		\nabla_\omega \widehat{\ell}(\mathcal{D}^b_{\text{test}}; \omega^b(\omega))= \frac{d \omega^b(\omega)}{d \omega} \cdot P^b
	\end{equation}
	with \begin{equation}\label{eqn:pb}
		P^b:=\nabla_\psi \widehat{\ell}(\mathcal{D}^b_{\text{test}} ; \psi)|_{\psi=\omega^b(\omega)}
	\end{equation}
	which can be efficiently computed via automatic differentiation \cite{rajeswaran2019meta}.
	However, computing $\frac{d \omega^b(\omega)}{d \omega}$ is non-trivial since $\omega^b(\omega)$ is defined via an optimization process \eqref{eqn:inner_imaml}. One possible approach is to backpropagate through the iteration \eqref{eqn:inner_gd}, but this requires storing the full optimization trajectory, which can be memory-intensive or even infeasible for many steps or non-differentiable solvers.

	To address this challenge, the following lemma is introduced to efficiently (in both computation and memory) compute \eqref{eqn:meta_gradient}:
	\begin{lemma}[Implicit Function Lemma~\cite{rajeswaran2019meta}]\label{lmm:implicit}
		Define
		\begin{equation}
			Q^b:=I+\frac{1}{\gamma} \nabla_{\psi}^2 \widehat{\ell}(\mathcal{D}^b_{\text{tr}}; \psi)|_{\psi=\omega^b(\omega)}.
		\end{equation} 
		If $Q^b$ is invertible, then 
		\begin{equation}\label{eqn:Qb}
			\frac{d \omega^b(\omega)}{d \omega}=(Q^b)^{-1}.
		\end{equation}
	\end{lemma}
	
	Combining \eqref{eqn:meta_gradient} and \eqref{eqn:Qb}, we obtain
	\begin{equation}\label{eqn:gradient}
		\nabla_\omega \widehat{\ell}(\mathcal{D}^b_{\text{test}}; \omega^b(\omega)) = (Q^b)^{-1} P^b.
	\end{equation}
	Since \( Q^b \) depends only on \( \omega^b(\omega) \), Lemma~\ref{lmm:implicit} allows efficient computation of \( \frac{d \omega^b(\omega)}{d \omega} \) without storing the entire optimization path.  
	However, obtaining the exact solution $\omega^b(\omega)$ of the inner problem \eqref{eqn:inner_imaml} requires full convergence, which is often impractical. In practice, we approximate \( \omega^b(\omega) \) using a few steps of gradient descent \eqref{eqn:inner_gd}.
	
		In practice, explicitly computing the inverse of \( Q^b \) can also be costly, especially in large-scale systems. To avoid matrix inversion, one can directly find \( (Q^b)^{-1}P^b \) 
		by solving the following optimization problem:
		\begin{equation}\label{eqn:cg}
			g^b :=\argmin_{\phi} \; \phi^T Q^b \phi - \phi^T P^b.
		\end{equation}
		Thus, $(Q^b)^{-1}P^b$ can be efficiently approximated by finding $g^b$ using an iterative method. 
	
	\begin{remark}\label{rmk:maml}
		Besides iMAML used here, another widely used meta-learning algorithm is MAML \cite{finn2017model}. Instead of \eqref{eqn:inner_imaml}, the inner loop of MAML solves
		\begin{equation}\label{eqn:inner_maml}
			\omega^b(\omega) := \argmin_{\psi} \; \widehat{\ell}(\mathcal{D}^b_{\text{tr}}; \psi).
		\end{equation}
		This formulation omits the regularization term in \eqref{eqn:inner_imaml}, which makes Lemma~\ref{lmm:implicit} inapplicable. As a result, the inner loop must compute and propagate derivatives of the training loss along the full optimization trajectory. To mitigate the high complexity, MAML employs a first-order approximation. However, these approximations can oversimplify the problem, potentially degrading the learning quality, as we will demonstrate through simulations in Section~\ref{sec:sim}.
		
	\end{remark}

\revise{
       Note that traditional offline learning approaches are often limited by structural assumptions. However, this limitation does not apply to the offline meta-training method proposed in this section. Specifically, our approach does not rely on offline learning to directly solve the target task; instead, it is embedded within a meta-learning framework to learn a transferable prior that facilitates rapid adaptation to new structural scenarios. By training on a distribution of tasks with varying system dynamics and structural characteristics, the learned meta-parameters capture common structures across scenarios. When applied to a target system, the method performs online adaptation using only a small amount of target data, significantly reducing data collection costs while maintaining flexibility to structural variations.
 }

	\section{Meta-Learning for Control}\label{sec:indirect}
	To achieve the control objective, we develop both indirect and direct data-driven methods within the meta-learning framework introduced in Section~\ref{sec:meta_learning}. The key difference between these approaches lies in whether an explicit system model is required, which directly influences the design of the loss functions for the base-learner and meta-learner. Particularly, we propose specific formulations of $\widehat{\ell}(\cdot)$ for each approach to achieve the desired objective \eqref{eqn:objective}. These algorithms highlight the generalizability of our framework, demonstrating that various controllers can be seamlessly integrated.

	\subsection{Indirect meta-learning for control}\label{sec:mpc}
	In the indirect data-driven approach, a common system model is learned during the meta-training phase. During meta-adaptation, the model is rapidly adapted using target data to approximate the target system dynamics, after which a model-based controller is applied to compute the optimal control.
	
	\subsubsection{Preliminaries: NSSM and MPC}
	To identify the system dynamics in~\eqref{eqn:sys}, various representations can be employed, such as neural differential equations and Hammerstein-Wiener models \cite{wills2013identification,chen2018neural}. Although our framework is agnostic to the choice of representation, we illustrate the controller design by identifying the system with an NSSM, which extends traditional state-space models by leveraging neural networks to capture nonlinearity: 
	\begin{subequations}\label{eqn:SSM}
		\begin{align}
			z_t &= f_{\text{enc}}(U_{t-H+1:t}, Y_{t-H+1:t}), \label{eqn:z}\\
			z_{t+1} &= A_z z_t + B_z u_{t+1}, \label{eqn:z_next}\\
			\hat{y}_t &= C_z z_t, \label{eqn:haty}
		\end{align}
	\end{subequations}
	where $z_t \in \mathbb{R}^{n_z}$ is a latent state encoded from past inputs and outputs:
	\begin{equation}
		\begin{split}
			U_{t-H+1:t}&:= \{u_{t-H+1},u_{u-H+2},\cdots,u_{t}\},\\
			Y_{t-H+1:t}&:= \{y_{t-H+1},y_{t-H+2},\cdots,y_{t}\}.
		\end{split}
	\end{equation}
	\revise{Here, $n_z$ and $H$ are user-defined design parameters. Since the encoder $f_{\text{enc}}$ maps a history of length $H$ to the latent state $z_t$, the dimension of $z_t$ must be sufficient to encode all relevant information from the input-output history.
	A practical rule of thumb is:
	$
	n_z \lesssim mH,
	$
	where $m$ is the output dimension. This ensures that the latent state can capture the essential dynamics contained in the past $H$ observations without being excessively large. Note that in \eqref{eqn:SSM}, $f_{\text{enc}}$ captures the system nonlinearity by fixing the initial condition of a linear system using past observations. The linear dynamics~\eqref{eqn:z_next}--\eqref{eqn:haty} then model temporal evolution.}
	
	Training NSSMs involves optimizing parameters $\omega := \{f_{\text{enc}}, A_z, B_z, C_z\}$ to minimize the identification error. Given a dataset of length $H+T$, the procedure is:
	\begin{enumerate}
		\item Construct a dataset $\mathcal{D}$ from input-output trajectories of length $H+T$:
		\begin{equation}\notag
			\begin{split}
				U_{t-H+1:t+T}&:= \{u_{t-H+1},u_{u-H+2},\cdots,u_{t+T}\},\\
				Y_{t-H+1:t+T}&:= \{y_{t-H+1},y_{t-H+2},\cdots,y_{t+T}\}.
			\end{split}
		\end{equation}
		\item Use past data $U_{t-H+1:t}$ and $Y_{t-H+1:t}$ to estimate the latent state $z_t$ via~\eqref{eqn:z}.
		\item Predict future outputs $\hat{Y}_{t+1:t+T} = \{\hat{y}_{t+1}, \dots, \hat{y}_{t+T}\}$ recursively using~\eqref{eqn:z_next} and~\eqref{eqn:haty}.
		\item Evaluate the identification accuracy using the loss:
		\begin{equation}\label{eqn:lssm}
			\ell_{\text{SSM}}(\mathcal{D}; \omega) = \frac{1}{T} ||Y_{t+1:t+T} - \hat{Y}_{t+1:t+T}||^2.
		\end{equation}
	\end{enumerate}
	
	Since the NSSM captures the input-output dynamics of the system, it enables the design of model-based controllers. To this end, we rewrite~\eqref{eqn:z_next} and~\eqref{eqn:haty} in compact form:
	\begin{equation}\label{eqn:compact}
		s_{t+1} = A s_t + B u_{t+1},
	\end{equation}
	where
	\[
	s_{t+1} := \begin{bmatrix} z_{t+1} \\ \hat{y}_{t+1} \end{bmatrix}, \quad
	A := \begin{bmatrix} A_z & 0 \\ C_z A_z & 0 \end{bmatrix}, \quad
	B := \begin{bmatrix} B_z \\ C_z B_z \end{bmatrix}.
	\]
	
	To track a reference trajectory, various controllers can be used; \revise{we focus on MPC for illustration due to the following reasons:   
\begin{itemize}
    \item The NSSM structure in results in linear latent dynamics \eqref{eqn:compact}, which ensures that the MPC formulation remains convex and computationally tractable.
    \item MPC naturally accommodates input constraints and tracking objectives.
    \item The receding-horizon mechanism of MPC improves robustness against modeling inaccuracies, which is particularly beneficial during early adaptation stages.
\end{itemize}
Therefore, MPC is an implementable controller that aligns with the identified NSSM model and practical control requirements.} It computes the control input by solving the optimization problem
	\begin{equation}\label{eqn:mpc}
		\begin{split}
			\min _u \;&\tilde{y}_{t+N \mid t}^T \cdot \mathbb{P} \cdot \tilde{y}_{t+N \mid t}\\&+\sum_{k=0}^{N-1} (\tilde{y}_{t+k \mid t}^T \cdot \mathbb{Q} \cdot \tilde{y}_{t+k \mid t}+\Delta u_{t+k \mid t}^T \cdot \mathbb{R} \cdot \Delta u_{t+k \mid t}) \\
			\text { s.t. } \; &\eqref{eqn:compact} \text{ and } \revise{u_{t+k \mid t} \in \mathcal{U}, \; \tilde{y}_{t+k \mid t} \in \mathcal{Y},}\; k=0,\ldots, N-1,
		\end{split}
	\end{equation}
	with
	$
	\tilde{y}_{t+k|t} := \hat{y}_{t+k|t} - \bar{y}_{t+k},  
	\Delta u_{t+k|t} := u_{t+k|t} - u_{t+k-1|t}.
	$
	Here, $\bar{y}_{t+k}$ is the reference signal, $\mathbb{Q} \succeq 0$ and $\mathbb{R} \succ 0$ are tuning weights. The matrix 
	$\mathbb{P}$ stabilizes the system's closed-loop performance and is typically chosen as the solution to a discrete-time algebraic Riccati equation \cite{garcia1989model}.

	\subsubsection{An MPC-based meta-learning control}
	To address limited data in the target system, we first pretrain an aggregated NSSM using source system data, then fine-tune it on the target system to identify its dynamics before applying MPC for reference tracking. Based on the meta-learning framework in Section~\ref{sec:meta_learning}, we now present the loss function design for both the base-learner and meta-learner.
	
	\textbf{Meta-training phase}:
	In the scenario of indirect data-driven control, the meta-training phase aims to identify a general model that can represent a group of nonlinear systems by using the source datasets. In the inner problem, the base-learner evaluates the identification performance on each single source system leading to the loss function
	\begin{equation}\label{eqn:inner_loss_mpc}
		\widehat{\ell}(\mathcal{D}^b_{\text{tr}}; \psi) = \ell_{\text{SSM}}(\mathcal{D}^b_{\text{tr}};\psi),
	\end{equation}
	where $\ell_{\text{SSM}}(\cdot)$ is given in \eqref{eqn:lssm}.
	
	In the outer problem, the meta-learner evaluates the identification performance across different source systems. It solves the problem \eqref{eqn:outer_prob}, where 
	\begin{equation}\label{eqn:out_loss_mpc}
		\ell(\mathcal{D}^b_{\text{test}}; \omega^b(\omega)) = \ell_{\mathrm{SSM}}(\mathcal{D}^b_{\text{test}}; \omega^b(\omega)).
	\end{equation}
	
	Finally, we update the source dataset. Specifically, solving the MPC problem~\eqref{eqn:mpc} yields the optimal input $u^b_*$, and to balance exploration and exploitation, we adopt the $\epsilon$-greedy strategy. Namely, for a given $\epsilon\in(0,1)$, we select the optimal action with probability $1-\epsilon$ and a random action with probability 
	$\epsilon$:
	\begin{equation}\label{eqn:ub}
		u^b = \begin{cases}
			u^b_*,  \text{ with probability } 1-\epsilon,\\
			\tilde{u}^b\sim \mathcal{N}(0,\Sigma^b), \text{ with probability } \epsilon.
		\end{cases}
	\end{equation}
	Applying $u^b$ to the source system~\eqref{eqn:sys} produces an output $y^b$, and the resulting trajectory $(u^b, y^b)$ is added to $\mathcal{D}_{\text{source}}$. Note that this phase requires only source data and can thus be conducted offline.

	\textbf{Meta-adaptation phase}:
	Upon convergence or reaching the maximum number of iterations, meta-training yields the aggregated NSSM with parameters denoted by $\omega$ that represent the shared similarities among source systems. In the meta-adaptation phase, $\omega$ is rapidly adapted to capture the target system's dynamics.
	
	We adapt the aggregated NSSM to the target system by performing an update similar to \eqref{eqn:inner_imaml}, using the target dataset $\mathcal{D}_{\text{target}}$:
	\begin{equation}\label{eqn:infer}
		\omega^*:=\argmin_{\psi} \; \ell_{\mathrm{SSM}}(\mathcal{D}_{\text{target}}; \psi)+\frac{\gamma}{2}||\psi-\omega||^2.
	\end{equation}
	In practice, this is efficiently approximated with a few gradient steps, using only limited data:
	\begin{equation}\label{eqn:infer_gd}
		\omega^*\gets \omega^* -\beta_{\text {in }} \big(\nabla_\psi\ell_{\mathrm{SSM}}(\mathcal{D}_{\text{target}}; \psi)|_{\psi=\omega^*}+\gamma(\omega^*-\omega)).
	\end{equation}
	The adapted NSSM, parameterized by $\omega^*$, is then used for MPC-based reference tracking. Note that the accurate system identification is critical for indirect data-driven control. The complete procedure is summarized in Algorithm~\ref{alg:meta_training}. 
	
	\begin{algorithm}
		\small
		\caption{Indirect meta-learning-based control with NSSM and MPC implementation}\label{alg:meta_training}
		\begin{algorithmic}
			\REQUIRE $\omega \leftarrow$ randomly initialize parameters of the NSSM
			\REQUIRE $\mathcal{D}_{\text {source }},\mathcal{D}_{\text {target }} \leftarrow$ source dataset, target dataset 
			\REQUIRE Regularization strength $\gamma$, learning rates $\beta_{\text{in}},\beta_{\text{out}}$, the maximum numbers of iterations $K_{\text{train}}, K_{\text{adapt}}$ ($K_{\text{adapt}}\ll K_{\text{train}}$), and design parameters $H, n_z,\mathbb{Q}, \mathbb{R}, N$
			\ENSURE $\omega^{*} \leftarrow$ NSSM parameters obtained via identification on the target system
			\ENSURE $u_t^{*} \leftarrow$ optimal control for the target system
			\STATE \% Meta-training phase
			\STATE $k=0$
			\WHILE[outer-loop]{not converge and $k < K_{\text{train}}$} 
			\STATE Sample batch $\{\mD^b\}_{b=1}^B$ from $\mathcal{D}_{\text {source }}$ 
			\FOR[inner-loop]{$b=1:B$}
			\STATE Partition $\mD^b$ into $\mathcal{D}^b_{\text{tr}}$ and $\mathcal{D}^b_{\text{test}}$
			\STATE $\omega^b \leftarrow$ solve \eqref{eqn:inner_imaml} with loss function \eqref{eqn:inner_loss_mpc}
			\STATE $g^b \leftarrow$ solve \eqref{eqn:cg} with $Q^b, P^b$ defined by the loss function \eqref{eqn:out_loss_mpc}
			\STATE $u^b\leftarrow$ compute via \eqref{eqn:mpc} and \eqref{eqn:ub} based on the NSSM parameterized by $\omega^b$
			\STATE $y^b\leftarrow$ simulate the system \eqref{eqn:sys} using $u^b$ with $\theta=\theta^b$  
			\STATE Update $\mathcal{D}_{\text{source}}$ by adding $(u^b,y^b)$
			\ENDFOR
			\STATE $\omega \leftarrow \omega-\beta_{\text {out }} \frac{1}{B}\sum_{b=1}^Bg^b$
			\STATE $k\gets k+1$
			\ENDWHILE
			\STATE \% Meta-adaptation phase
			\WHILE{$k < K_{\text{adapt}}$} 
			\STATE $\omega^* \leftarrow$ update $\omega^*$ by \eqref{eqn:infer_gd}
			\ENDWHILE
			\STATE \% MPC to find the optimal control input
			\FOR{$t=1, 2,\cdots$}  
			\STATE $u_t^*\leftarrow $ solve MPC \eqref{eqn:mpc} using the NSSM parameterized by $\omega^*$
			\ENDFOR
		\end{algorithmic}
	\end{algorithm}

    \begin{remark}
    \revise{Stability guarantees are a fundamental issue in optimal control. Specifically, since the controller is designed based on a meta-learned model, the closed-loop performance depends directly on the quality of the learning stage. To clarify this relationship, our analysis establishes the following pathway:
\begin{enumerate}
    \item[(1)] The meta-training phase converges under standard smoothness and regularization assumptions and produces a well-defined shared initialization.
    \item[(2)] The subsequent task-specific adaptation in the meta-inference phase identifies the target system with a bounded learning error.
    \item[(3)] When the adapted model is embedded within an MPC framework, robust MPC theory guarantees input-to-state stability under bounded model mismatch.
\end{enumerate}
Therefore, stabilization directly follows from the combination of bounded learning error and standard robust MPC results.  As compared to supervised learning without pretraining stage, meta-learning reduces adaptation error, which in turn tightens stability and tracking bounds in closed-loop control, as will be observed from the simulation results in Section~\ref{sec:sim}. For readability, the detailed theoretical analysis is provided in Appendix~\ref{app:theory}. }
    \end{remark}
	
	%%%%%%%%%%%%%%%%%%%%%%%%%%%%%%%%%%%%%%%%%%%%%%%%%%%%%%%%%%%%%%%%%%%%%%
	\subsection{Direct meta-learning for control}\label{sec:direct}
    \revise{The method proposed in Section~\ref{sec:indirect} belongs to the category of indirect data-driven control. Such approaches are modular and naturally compatible with constraint handling through MPC. However, their performance may degrade due to model bias and identification errors. Moreover, the identified model may not be optimal for control design, for example due to mismatched uncertainty representations \cite{dorfler2022bridging}. Direct approaches aim to address these limitations by learning control policies directly from data in an end-to-end manner. Reinforcement learning (RL), as one of the most widely adopted end-to-end paradigms, is therefore introduced to provide a direct data-driven realization of the same optimal control objective when accurate system models are subject to significant uncertainty.}
    
    Here we adopt a standard RL algorithm, the deep Q-network (DQN) \cite{Mnih2015HumanlevelCT,gautier2022deep} as the base-learner. 
    \revise{The conceptual simplicity of DQN makes the integration with our meta-learning framework transparent, and it has well-established convergence properties under appropriate conditions \cite{Lisboa0Convergence}. Importantly, our framework is not limited to DQN. The bi-level optimization structure \eqref{eqn:inner_imaml} - \eqref{eqn:outer_prob} requires only that the loss function $\widehat\ell(\cdot)$ is  differentiable. More advanced RL algorithms, such as soft actor-critic methods \cite{haarnoja2018soft}, proximal policy optimization \cite{schulman2017proximal}, and model-based RL approaches \cite{m2023model}, can be integrated by defining $\widehat\ell(\cdot)$ as their policy gradient or value function losses.}

	\subsubsection{A DQN-based RL reformulation}
	
	To apply DQN, we reformulate the reference tracking task as an RL problem, \revise{where input and output constraints are naturally embedded in the Markov Decision Process (MDP) formulation by restricting the admissible state and action spaces.}
    During meta-training, a so-called Q-network is pretrained to capture the similarities across similar systems. In the meta-adaptation phase, the network is fine-tuned with limited target data to directly learn an optimal control policy. To achieve reference tracking, the objective is to minimize the discounted cumulative cost
	$
	\sum_{k=0}^{\infty} \gamma^t c(y_{t}, u_{t}), 
	$  
	where $\gamma \in (0,1)$ is the discount factor that balances immediate and future losses. Moreover, $c(y_t, u_{t})$ represents the stage cost at time $t$, defined as  
	\begin{equation}  
		c(y_t, u_{t}) := (y_t - \bar{y}_t)^T  \mathbb{Q}(y_t - \bar{y}_t) + \Delta u_t^T \mathbb{R} \Delta u_t 
	\end{equation}
	with $\Delta u_t:=u_{t}-u_{t-1}$.
	
	Let $\mu(y)$ be the action taken under $y$.
	We can define the action-value function $Q^\mu(y, u)$ that takes the current action into account:
	\begin{equation}
		\begin{split}
			&Q^\mu(y, u)=\sum_{t=0}^{\infty} \gamma^t c(y_{t},u_{t}),\\
			\text{where }   &y_0=y, u_1=u, u_t=\mu(y_t), t=1,...,\infty. 
		\end{split}
	\end{equation}
	The optimal action-value function over all possible actions is denoted as
	\begin{equation}
		Q^*(y, u) = \min_\mu Q^\mu(y, u),
	\end{equation}
	which follows the well-known Bellman equation \cite{Bertsekas:DPOC:Vol1}:
	\begin{equation}\label{eqn:Q*}
		Q^*(y, u)=c(y,u)+\min_{u'} \gamma Q^*\left(y', u'\right). 
	\end{equation}
	Here $y'$ denotes the next state of $y$, following the system dynamics~\eqref{eqn:sys} after applying the control input $u'$.
	If $Q^*$ is available we can compute the optimal policy by
	\[
	\mu(y)= \operatorname{argmin}_u Q^*(y, u). 
	\]
	However, since both $y$ and $u$ are real valued, computing $Q^*$ involves solving equation \eqref{eqn:Q*} in the infinite dimensional space of real-valued functions. As this is practically impossible, we restrict attention to a finitely parameterised class of functions $Q(\cdot, \cdot|\omega):\mathbb{R}^n\times \mathcal{C}\to \mathbb{R}$, where $\omega$ represents the adaptable parameter. In DQN, this is achieved using a neural network known as the Q-network. Suppose that a dataset $\mathcal{D} := \{U(T), Y(T)\}$, containing a $T$-length trajectory of system inputs and outputs, is available for evaluation. We first rearrange it as
	\[
	\mathcal{D} = \{(y^{(i)}, u^{(i)}, y'^{(i)})\}_{i=1}^T,
	\]
	that is, each data tuple consists of the output $y^{(i)}$, the input $u^{(i)}$, and the subsequent output $y'^{(i)}$.
	Inspired by \eqref{eqn:Q*}, the Q-network is trained by adjusting its weights $\omega$ to reduce the mean-squared Bellman error
	\begin{equation}
		\ell_{\text{DQN}}(\mD;\omega)=\frac{1}{2}\sum_{i=1}^T(\tilde{Q}(y^{(i)}, u^{(i)}|\omega_-)-Q(y^{(i)}, u^{(i)}|\omega))^2,
	\end{equation}
	where
	\begin{equation}\label{eqn:loss}
		\begin{split}
			&\tilde{Q}(y^{(i)}, u^{(i)}|\omega_-)=c(y^{(i)},u^{(i)})+\min_{u'} \gamma Q(y'^{(i)}, u'|\omega_-).
		\end{split}
	\end{equation}
	In this equation, the optimal target value $c(y^{(i)},u^{(i)})+\min_{u'} \gamma Q^*(y'^{(i)}, u')$ is substituted with the approximated target value $c(y^{(i)},u^{(i)})+\min_{u'} \gamma Q(y'^{(i)}, u'|\omega_-)$, with parameter $\omega_-$ from the previous iteration. At each stage of optimization, we hold the parameters from the previous iteration $\omega_-$ fixed when optimizing the loss function $\ell_{\text{DQN}}(\mD;\omega)$. 
	
	Differentiating the loss function \eqref{eqn:loss} with respect to the weights yields the following gradient:
	\begin{equation}\label{eqn:gradient_rl}
		\begin{split}
			&\nabla_{\omega} \ell_{\text{DQN}}(\mD;\omega)\\\!\!=&\!-\!\!\sum_{i=1}^T(\tilde{Q}(y^{(i)}, u^{(i)}|\omega_-)\!-\! Q(y^{(i)}, u^{(i)}|\omega)) \nabla_{\omega} Q(y^{(i)}, u^{(i)}|\omega).\\    
		\end{split}
	\end{equation}
	We then update $\omega$ by minimizing $\ell_{\text{DQN}}(\mD;\omega)$ using the gradient descent algorithm:
	\begin{equation}
		\omega  \leftarrow \omega -\alpha \nabla_\omega \ell_{\text{DQN}}(\mD;\omega).
	\end{equation}
	As suggested by \cite{Lisboa0Convergence}, the Q-network is further updated by
	\begin{equation}\label{eqn:omega-}
		\omega_- \leftarrow \beta \omega_- + (1-\beta) \omega,
	\end{equation}
	where $\beta\in [0.5, 1)$. 
	
	\begin{remark}\label{rmk:dqn}
		\iffalse
        Recall the system dynamics in \eqref{eqn:sys}. If $g(\cdot)$ is an injective function, the state $x$ can be uniquely determined by the output $y$, and the DQN yields the globally optimal solution. Otherwise, when the system \eqref{eqn:sys} is observable, one can stack the output vector
		$
		\mathcal{Y}_t = [ y_{t-n-1}, \ldots, y_t]
		$
		that fully determines the state $x_t$. For this case, $y$ in this subsection can be replaced by $\mathcal{Y}$ to obtain the global optimal solution. The remaining derivations and analysis follow analogously.
		\fi
        \revise{The convergence of DQN to the optimal action-value function $Q^*$ requires the Markov property: the information available at each decision point must fully determine the system state. For system \eqref{eqn:sys}, when $g(\cdot)$ is injective, the current output $y_t$ uniquely determines the state $x_t$, and we use $y_t$ as input to the Q-network. When $g(\cdot)$ is not injective, we use an output history $Y_t = [y_{t-n-1}, \ldots, y_t]$ as the Q-network input. For observable systems, this output history fully reconstructs the state, satisfying the Markov property. Under these conditions and standard RL assumptions (adequate exploration, appropriate learning rate schedules \cite{Lisboa0Convergence}), DQN provably converges to the optimal $Q^*$. Throughout this subsection, the derivations apply with $y$ replaced by $Y$ when the non-injective case is considered. Here, ``optimal'' indicates that the solution satisfies the Bellman equation \eqref{eqn:Q*}, and convergence can be achieved when the Q-network gradients are exactly known. Thus, practical neural network approximation may introduce errors.}
	\end{remark}

	\subsubsection{A DQN-based meta-learning control}
	To enable rapid adaptation in dynamic environments, we embed DQN within the proposed meta-learning framework. During the meta-training phase, a shared parameter $\omega$ is learned from different source systems using $\mathcal{D}_{\text{source}}$. Subsequently, the meta-adaptation phase fine-tunes $\omega$ to determine the optimal control policy for the target system, as specified below.
	
	\textbf{Meta-training phase}:
	By using DQN, the base-learner for each source system, labeled as $b$, solves the inner problem \eqref{eqn:inner_imaml} to yield $\omega^b(\omega)$ with
	\begin{equation}\label{eqn:inner_loss_rl}
		\widehat{\ell}(\mathcal{D}^b_{\text{tr}}; \psi) = \ell_{\text{DQN}}(\mathcal{D}^b_{\text{tr}};\psi).
	\end{equation}
	On the other hand, in the outer problem, the meta-learner solves \eqref{eqn:outer_prob}, where
	\begin{equation}\label{eqn:out_loss_rl}
		\ell(\mathcal{D}^b_{\text{test}}; \omega^b(\omega)) = \ell_{\text{DQN}}(\mathcal{D}^b_{\text{test}}; \omega^b(\omega)).
	\end{equation}
	As shown in \eqref{eqn:inner_gd} and \eqref{eqn:pb}, this requires evaluating the gradient \eqref{eqn:gradient_rl}.
	
	Similar to \eqref{eqn:ub}, in the meta-training phase, the control input $u^b$ is selected using the $\epsilon$-greedy policy that balances the exploration and exploitation:
	\begin{equation}\label{eqn:ub_dqn}
		u^b = \begin{cases}
			\operatorname{argmin}_u Q(y^b, u,\omega^b(\omega)),  \text{ with probability } 1-\epsilon,\\
			\tilde{u}^b\sim \mathcal{N}(0,\Sigma^b), \text{ with probability } \epsilon.
		\end{cases}
	\end{equation} That is, the $\epsilon$-greedy policy generates new sample data based on latest $\omega^b$.

	\begin{algorithm}
		\small
		\caption{Direct meta-learning-based control with  DQN implementation}\label{alg:meta_training_RL}
		\begin{algorithmic}
			
			\REQUIRE $\omega \leftarrow$ randomly initialize the weights to parameterize the optimal action-value function
			\REQUIRE  $\mathcal{D}_{\text {source }},\mathcal{D}_{\text {target }} \leftarrow$ source dataset, target dataset 
			\REQUIRE Regularization strength $\gamma$, learning rates $\beta_{\text{in}},\beta_{\text{out}}$, the maximum numbers of iterations $K_{\text{train}}, K_{\text{adapt}}$ ($K_{\text{adapt}}\ll K_{\text{train}}$), and design parameters $\mathbb{Q}, \mathbb{R},\alpha,\beta,\gamma$
			\ENSURE $\omega^{*} \leftarrow$ weights of the optimal action-value function for the target system
			\ENSURE $u_t^{*} \leftarrow$ optimal control for the target system
			\STATE \% Meta-training phase
			\STATE $k=0$
			\WHILE[outer-loop]{not converge and $k < K_{\text{train}}$} 
			\STATE Sample batch $\{\mD^b\}_{b=1}^B$ from $\mathcal{D}_{\text {source }}$ 
			\FOR[inner-loop]{$b=1:B$}
			\STATE Partition $\mD^b$ into $\mathcal{D}^b_{\text{tr}}$ and $\mathcal{D}^b_{\text{test}}$
			\STATE $\omega^b \leftarrow$ solve \eqref{eqn:inner_imaml} with loss function 
			\eqref{eqn:inner_loss_rl}
			\STATE $\omega_{-}^b \leftarrow$ update $\omega_{-}^b$ via \eqref{eqn:omega-}
			\STATE $g^b \leftarrow$ solve \eqref{eqn:cg} with $Q^b,P^b$ defined by the loss function \eqref{eqn:out_loss_rl}
			\STATE $u^b\leftarrow$ choose the new input by \eqref{eqn:ub_dqn}
			\STATE $y^b\leftarrow$ simulate the system \eqref{eqn:sys} using $u^b$ with $\theta=\theta^b$  
			\STATE Update $\mathcal{D}_{\text{source}}$ by adding $(u^b,y^b)$
			\ENDFOR
			\STATE $\omega \leftarrow \omega-\beta_{\text {out }} \frac{1}{B}\sum_{b=1}^Bg^b$
			\ENDWHILE
			\STATE \% Meta-adaptation phase
			\WHILE{$k < K_{\text{adapt}}$} 
			\STATE $\omega^* \leftarrow$ update $\omega^*$ by \eqref{eqn:infer_rl_gd}
			\ENDWHILE
			\STATE \% Directly find the control input
			\FOR{$t=1, 2,\cdots$}  
			\STATE $u_t^*\leftarrow $ solve \eqref{eqn:control_rl}
			\ENDFOR
		\end{algorithmic}
	\end{algorithm}

	\textbf{Meta-adaptation phase}:
	Once the meta-training phase converges or reaches its iteration limit, $\omega$ is learned as a general approximation of action-value functions across source systems. Similar to \eqref{eqn:infer_gd}, in the meta-adaptation phase, $\omega$ is fine-tuned using a small amount of data from the target system:
	\begin{equation}\label{eqn:infer_rl_gd}
		\omega^*\gets \omega^* -\beta_{\text {in }} \big(\nabla_\psi\ell_{\mathrm{DQN}}(\mathcal{D}_{\text{target}}; \psi)|_{\psi=\omega^*}+\gamma(\omega^*-\omega)).
	\end{equation}
	Then, during the real-time operation, the optimal control input is directly obtained by solving
	\begin{equation}
		u_t^*=\operatorname{argmin}_u Q(y_t, u,\omega^*). \label{eqn:control_rl}  
	\end{equation}

	The full procedure is summarized in Algorithm~\ref{alg:meta_training_RL}. Compared with Algorithm~\ref{alg:meta_training}, Algorithm~\ref{alg:meta_training_RL} does not aim to identify the system model. Instead, it implicitly learns the control policy by optimizing the action-value functions.
	
	As shown in this section, the framework proposed in Section~\ref{sec:meta_learning} can accommodate both indirect and direct learning algorithms for controller design. \revise{
From a structural perspective, indirect methods parameterize the system model and subsequently perform model-based control design, whereas direct methods parameterize the control policy (or value function) and optimize it directly from data.
From an optimization viewpoint, both approaches solve the same bi-level meta-learning problem. The difference lies in the choice of decision variables (model parameters versus policy parameters) and in how approximation errors propagate (model bias versus policy approximation bias). The motivation for studying both paradigms is therefore to demonstrate that the proposed framework is sufficiently general to accommodate both realizations under a unified optimization structure.    
} It should also be emphasized that the applicability of the framework is not limited to the examples discussed in Sections~\ref{sec:mpc} and \ref{sec:direct}.

	\section{Simulation and Experiment}\label{sec:sim}
	In this section, we will illustrate the performance of our proposed algorithms by using both numerical simulations and experiments.
	
	\subsection{Numerical verification of Algorithm~\ref{alg:meta_training}}
	We first present numerical examples to evaluate the performance of Algorithm~\ref{alg:meta_training}.
	Let us consider a family of van der Pol oscillators whose dynamics are given by
	\begin{equation}
		\begin{split}
			\dot{x}_1&=x_2, \;
			\dot{x}_2=\theta x_2\left(1-x_1^2\right)-x_1+u, \\
			y&=x,
		\end{split}
	\end{equation}
	where $\theta\sim \mathcal{N}(0,1)$ is the unknown damping ratio. 
	% We set the latent state dimension in \eqref{eqn:SSM} to $n_z = 5$. The encoder $f_{\text{enc}}$ is a neural network with one input layer, one hidden layer, and one output layer, each with $128$ ReLU-activated neurons. The state-space matrices $A_z$, $B_z$, and $C_z$ are randomly initialized. All experiments are repeated with $10$ different random seeds.

	\revise{We compare Algorithm~\ref{alg:meta_training} with a few baselines, all of which have the same NSSM architecture as detailed in \cite{yan2024mpc}: 
	\begin{enumerate}
		\item \textbf{iMAML (Our approach)}: The NSSM parameters are pretrained on the source data $\mathcal{D}_{\text{source}}$ and subsequently adapted to the target data $\mathcal{D}_{\text{target}}$ using Algorithm~\ref{alg:meta_training}.
		\item \textbf{MAML}: \textbf{MAML} also employs a meta-learning framework. However, as noted in Remark~\ref{rmk:maml}, it solves a different bi-level optimization problem from \textbf{iMAML} and applies a first-order approximation.
		\item \textbf{Supervised learning}: A standard learning approach without pretraining. NSSM parameters are initialized randomly and trained solely on target data $\mathcal{D}_{\text{target}}$. 
	\end{enumerate}}
    \revise{Note that in both \textbf{iMAML} and \textbf{MAML}, $10$ source systems are generated by sampling $\theta^b \sim \mathcal{N}(0,1)$. The dataset of each source system contains a trajectory of length $10000$.}

    \revise{After pretraining (\textbf{iMAML} and \textbf{MAML}) or randomly initialization (\textbf{Supervised learning}), we adapt the NSSMs using the target dataset $\mathcal{D}_{\text{target}}$ containing $300$ data points. Fig.~\ref{fig:control} illustrates the tracking performance on the target system. The goal is to track a circular trajectory using MPC with NSSMs adapted after $0$, $100$, and $3000$ iterations of \eqref{eqn:infer_gd}, respectively. From the first column, we conclude that after the meta-training phase, both \textbf{iMAML} and \textbf{MAML} produce good aggregated models that capture the dynamics of the target system as evidenced by their ability to follow the overall trend of the reference trajectory from the start. The performance is further improved after the adaptation steps using target data $\mD_{\text{target}}$. Notably, \textbf{iMAML} shows faster convergence and superior tracking, particularly during early meta-adaptation. This is attributed to its more accurate computation of meta-gradient than \textbf{MAML} (see Remark~\ref{rmk:maml}), which leads to better initialization and more effective adaptation. In contrast, since \textbf{supervised learning} does not benefit from the source systems, it performs poorly at the beginning and converges slower. }
	
	% \begin{figure}[!tb]
	% 	\centering
	% 	\includegraphics[width=0.8\linewidth]{  log_loss2}
	% 	\caption{\textcolor{teal}{TO BE REMOVED}Comparison of prediction performance on the target system. The lines and shaded regions respectively represent the mean and standard deviation across $10$ runs.}\label{fig:loss_compare_infer}
	% \end{figure}

	\begin{figure*}[!tb]
		\centering
		\subfigure[The iMAML-based MPC (Algorithm~\ref{alg:meta_training}).]{
			\begin{minipage}[b]{0.3\textwidth}
				\includegraphics[width=1\textwidth]{  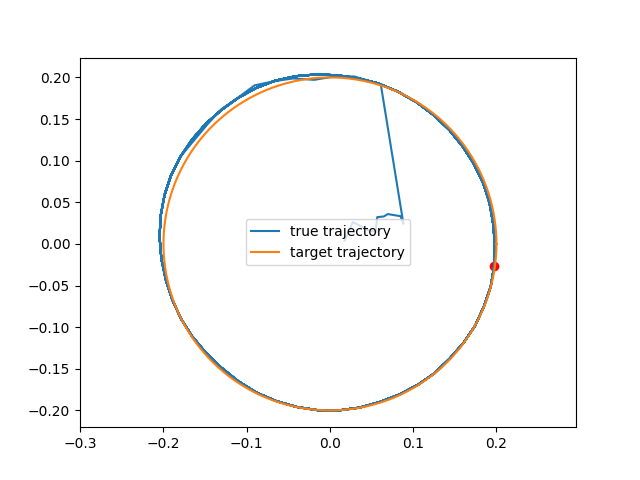} 
			\end{minipage}
			\begin{minipage}[b]{0.3\textwidth}
				\includegraphics[width=1\textwidth]{  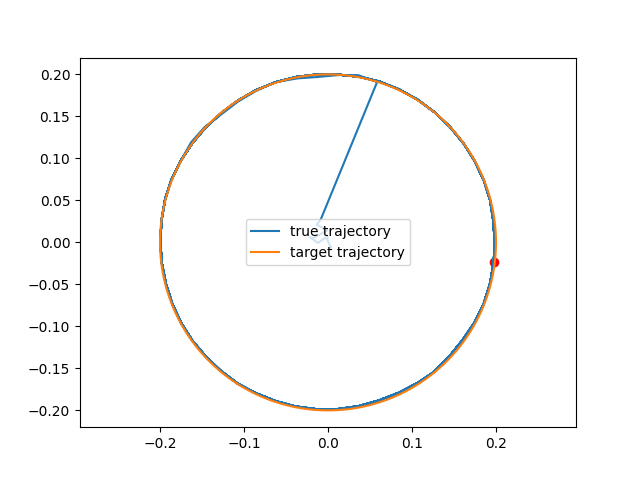}
			\end{minipage}
			\begin{minipage}[b]{0.3\textwidth}
				\includegraphics[width=1\textwidth]{  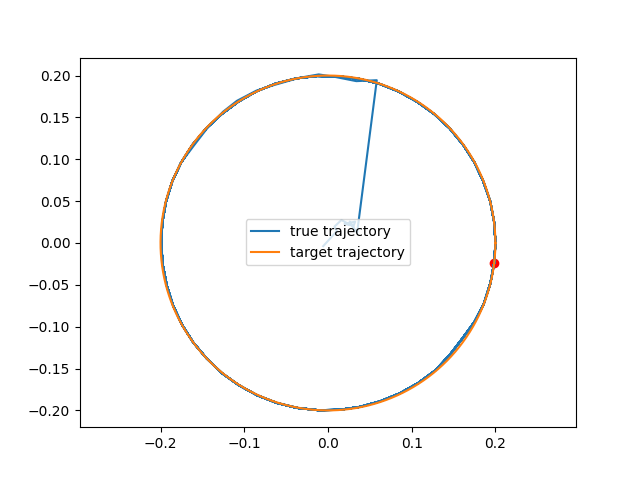}
			\end{minipage}
		}
		\subfigure[The MAML-based MPC.]{
			\begin{minipage}[b]{0.3\textwidth}
				\includegraphics[width=1\textwidth]{  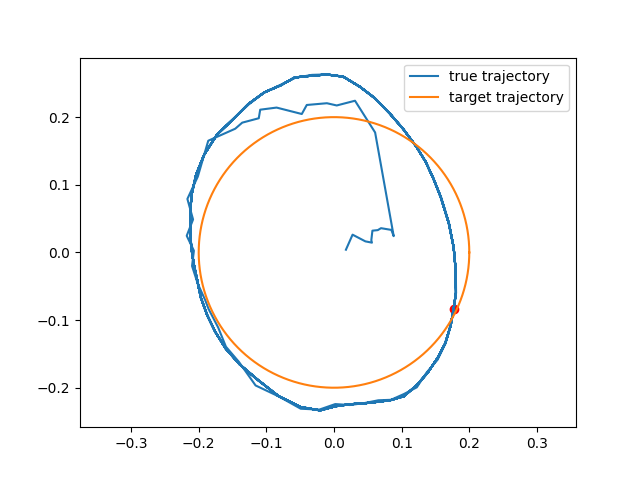} 
			\end{minipage}
			\begin{minipage}[b]{0.3\textwidth}
				\includegraphics[width=1\textwidth]{  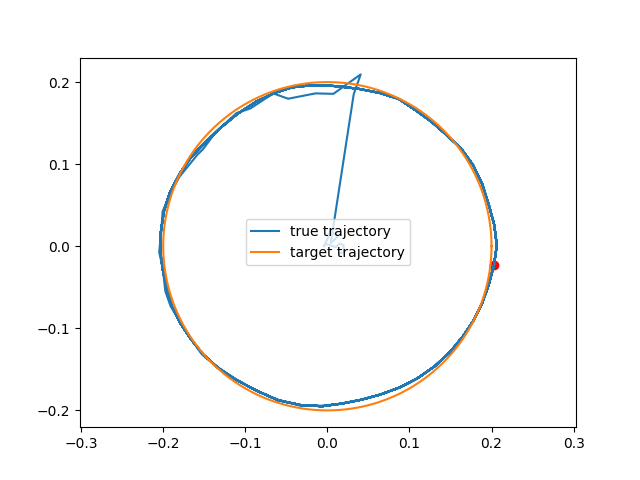}
			\end{minipage}
			\begin{minipage}[b]{0.3\textwidth}
				\includegraphics[width=1\textwidth]{  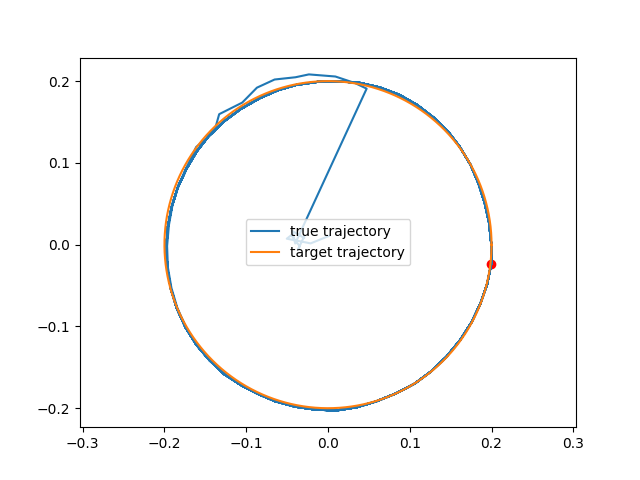}
			\end{minipage}
		}
		\subfigure[The supervised learning-based MPC.]{
			\begin{minipage}[b]{0.3\textwidth}
				\includegraphics[width=1\textwidth]{  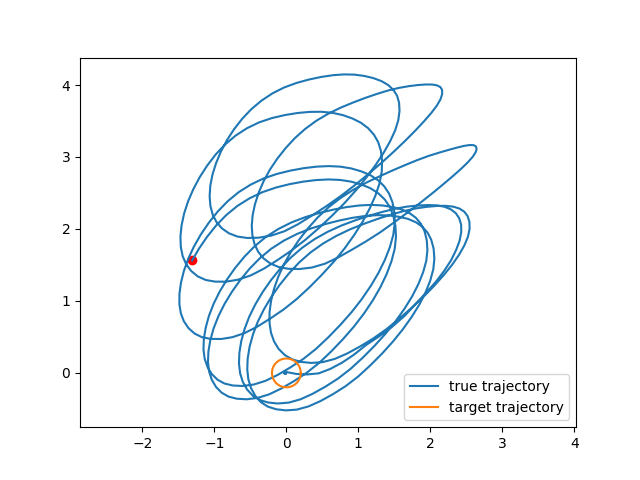} 
			\end{minipage}
			\begin{minipage}[b]{0.3\textwidth}
				\includegraphics[width=1\textwidth]{  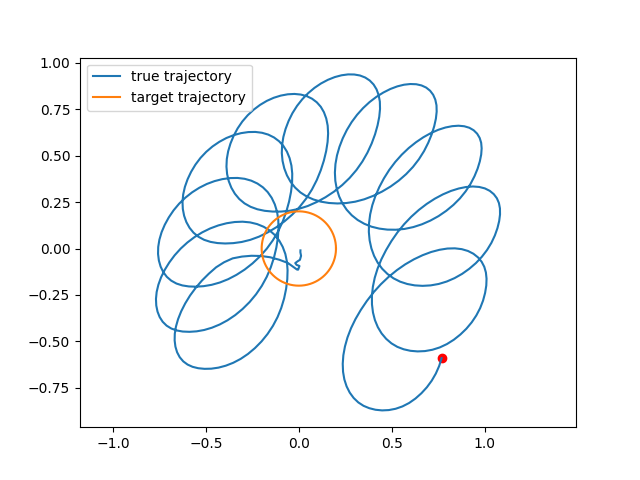}
			\end{minipage}
			\begin{minipage}[b]{0.3\textwidth}
				\includegraphics[width=1\textwidth]{  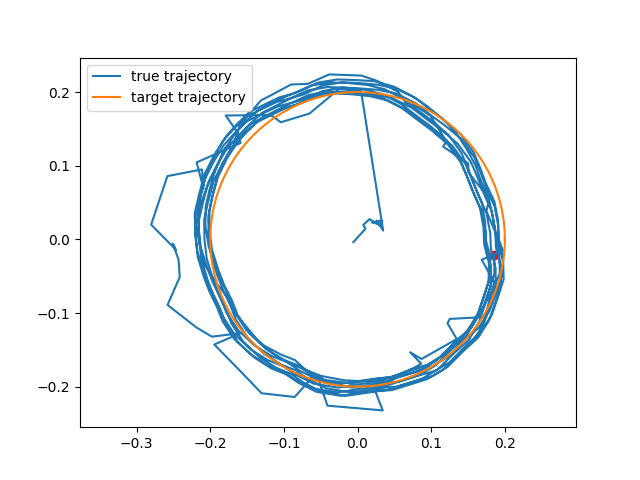}
			\end{minipage}
		}
		\caption{\revise{Comparison of the tracking performance on the target system, where the three columns show the performance of different algorithms by using the NSSMs adapted after $10$, $100$, and $3000$ steps, respectively. }}\label{fig:control}
	\end{figure*}

	\subsection{Experimental verification of Algorithm~\ref{alg:meta_training_RL}}
	To evaluate the performance of the direct meta-learning-based controller, we test Algorithm~\ref{alg:meta_training_RL} on a physical ball-on-a-plate system, as shown in Fig.~\ref{fig:ball}. The system comprises a square plate, actuated by two servo motors that control the plate’s rotation about the 
	$x$- and $y$-axes. A ball rests in the plate and can be moved by tilting the plate correspondingly. A camera mounted over the plate can be used to track the ball position. The control loop is closed through a raspberry pi board.

	\begin{figure}[!tb]
		\centering
		\includegraphics[width=\linewidth]{   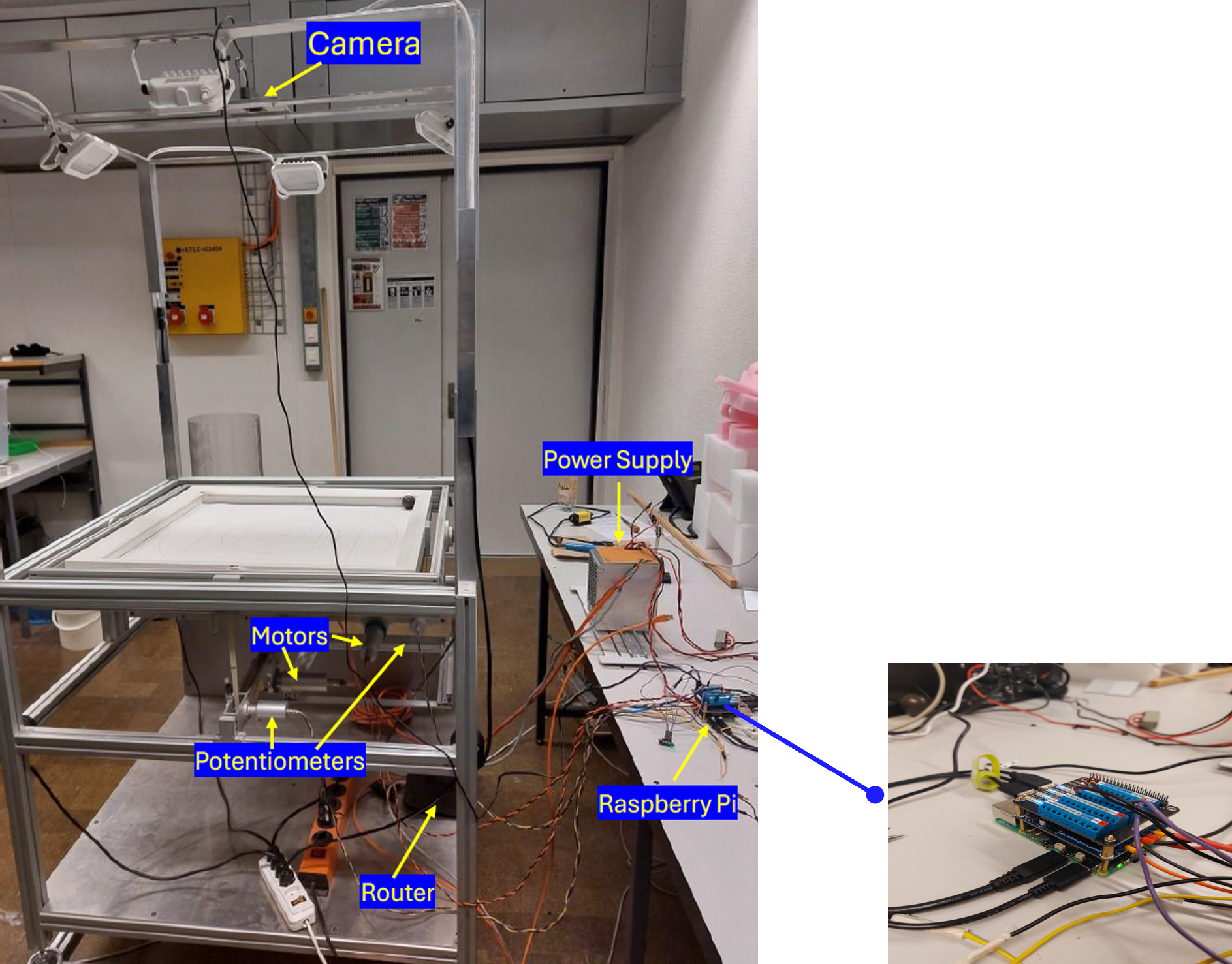}
		\caption{The ball-on-a-plate system \cite{Simon}.} \label{fig:ball}
	\end{figure}
	
	The system is nonlinear with uncertainties introduced by friction. The dynamics can be described by \cite{Waldvogel} and \cite{Simon}:
	\begin{equation}
		\begin{gathered}
			\ddot{x}=\frac{5}{7} g \sin \beta \cos \alpha+\frac{F_{\text {friciton},x}}{m}, \\
			\ddot{y}=-\frac{5}{7} g \sin \alpha+\frac{F_{\text {friciton},y}}{m},
		\end{gathered}  
	\end{equation}
	where $x$ and $y$ represent the ball's position along the $x$- and $y$-axes; $\alpha$ and $\beta$ denote the plate angles of rotation about the $x$- and $y$-axes; $F_{\text{friction},x}$ and $F_{\text{friction},y}$ are the frictional forces projected along the $x$- and $y$-directions; and $m$ is the mass of the ball.

	We adopt the classical Stribeck curve as the friction model, following the formulation in \cite{Waldvogel}. The mathematical expression is given by
	$$
	F_{\text {friciton }}=\left(F_C+\left(F_S-F_C\right) e^{-\left(\frac{||v||}{v_S}\right)^{\delta_S}}\right) \operatorname{sgn}(v)+F_v v,
	$$
	where $F_C$ is the kinetic friction coefficient, $F_S$ is the static friction coefficient, $F_v$ denotes the viscous friction coefficient, $v_S$ represents the Stribeck velocity, $\delta_S$ depends on the contact surface geometry, and $v = [\dot{x}_s, \dot{y}_s, 0]^T$ is the ball's velocity. A graphical illustration of these parameters is provided in Fig.~\ref{fig:friction}. 
	
	Since the true friction parameters are unknown and collecting data from the real system is challenging, we create $10$ source systems in simulation. These systems share the same model as the real system but have different, unknown friction parameters.
	
	\begin{figure}[!tb]
		\centering
		\includegraphics[width=0.6\linewidth]{   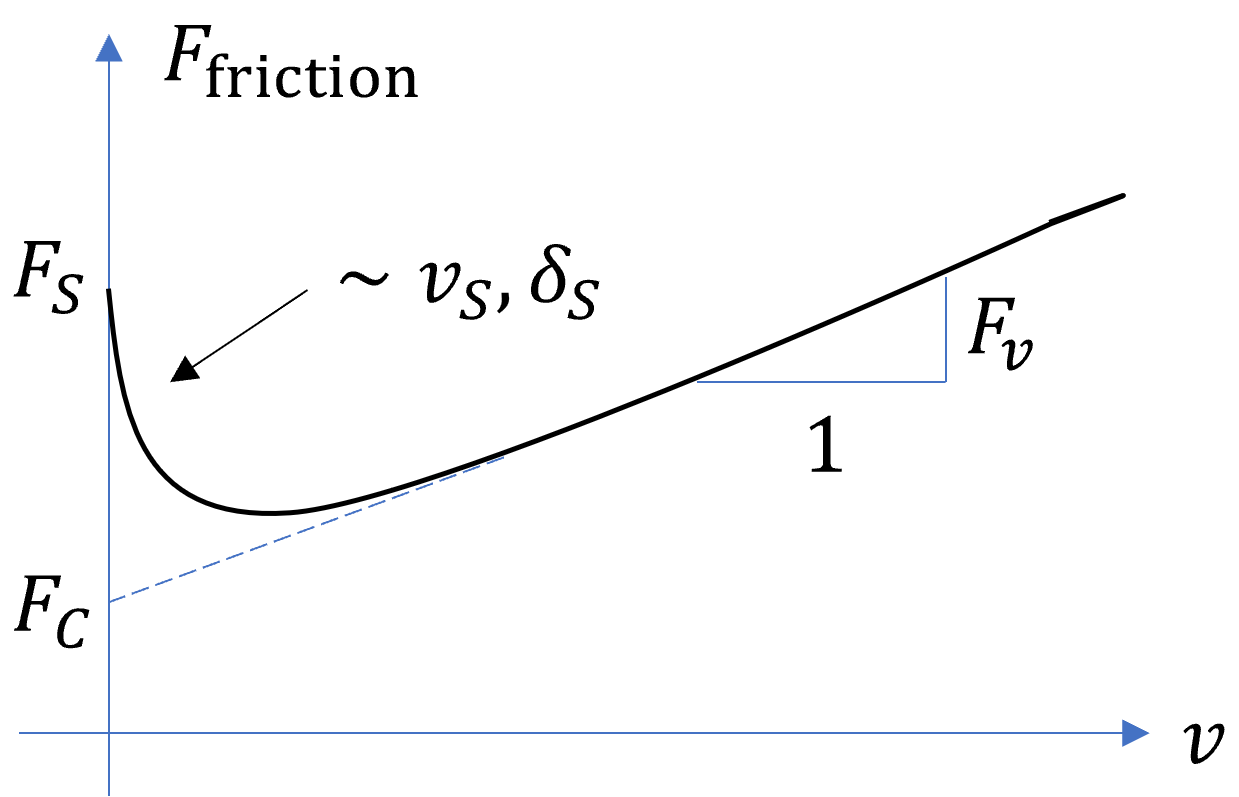}
		\caption{Velocity-dependent friction model.}
		\label{fig:friction}
	\end{figure}
	
	To evaluate the effectiveness of our algorithms, we consider two scenarios:
	\begin{enumerate}
		\item The Q-network is first pretrained in simulation using data from source systems and subsequently adapted to the real physical system using Algorithm~\ref{alg:meta_training_RL}.
		\item A DQN is trained directly on the real system with randomly initialized weights.
	\end{enumerate}
	The learning algorithms compute and send the desired plate angles in the $x$- and $y$-directions based on the ball’s current state. These angles serve as control inputs that ultimately adjust the ball’s position accordingly.
	
	For exploration, in addition to the $\epsilon$-greedy policy used in Algorithm~\ref{alg:meta_training_RL}, we introduce an ``input intermezzo" mechanism. That is, after a fixed operating period, we intentionally override the current policy by commanding a predefined plate angle based on the ball's reference position. The input intermezzo is used to repeatedly reset the ball to positions, which are otherwise rarely attained under $\epsilon$-greedy actions paired with a well-trained tracking controller. This, in particular, includes collisions with the plate boundary. The input intermezzo thereby empirically achieved convergence speed.

	Figures~\ref{fig:ball_imaml} and \ref{fig:sl} compare the tracking performance on the real system with and without pretraining. The plots show the reference signal, which follows a square trajectory (green line), alongside the actual position of the ball (blue line). The short intervals between the dotted red lines represent input intermezzos used for exploration. It is evident that pretraining the network weights using Algorithm~\ref{alg:meta_training_RL} offers a significant advantage: the overall trend of the reference trajectory is captured from the very beginning (Fig.~\ref{fig:ball_imaml}). Moreover, over the 
	$150$-second experiment, the pretrained weights lead to a 
	$73.58\%$ reduction in mean cost compared to direct deployment on the real system without any pretraining (Fig.~\ref{fig:sl}). This confirms that the meta-learning framework enhances tracking performance in real-world applications, highlighting its effectiveness for sim-to-real transfer. Moreover, while oscillations are observable in the results, these are expected to diminish with more training data.

	\begin{figure}[!tb]
		\centering
		\includegraphics[width=\linewidth]{   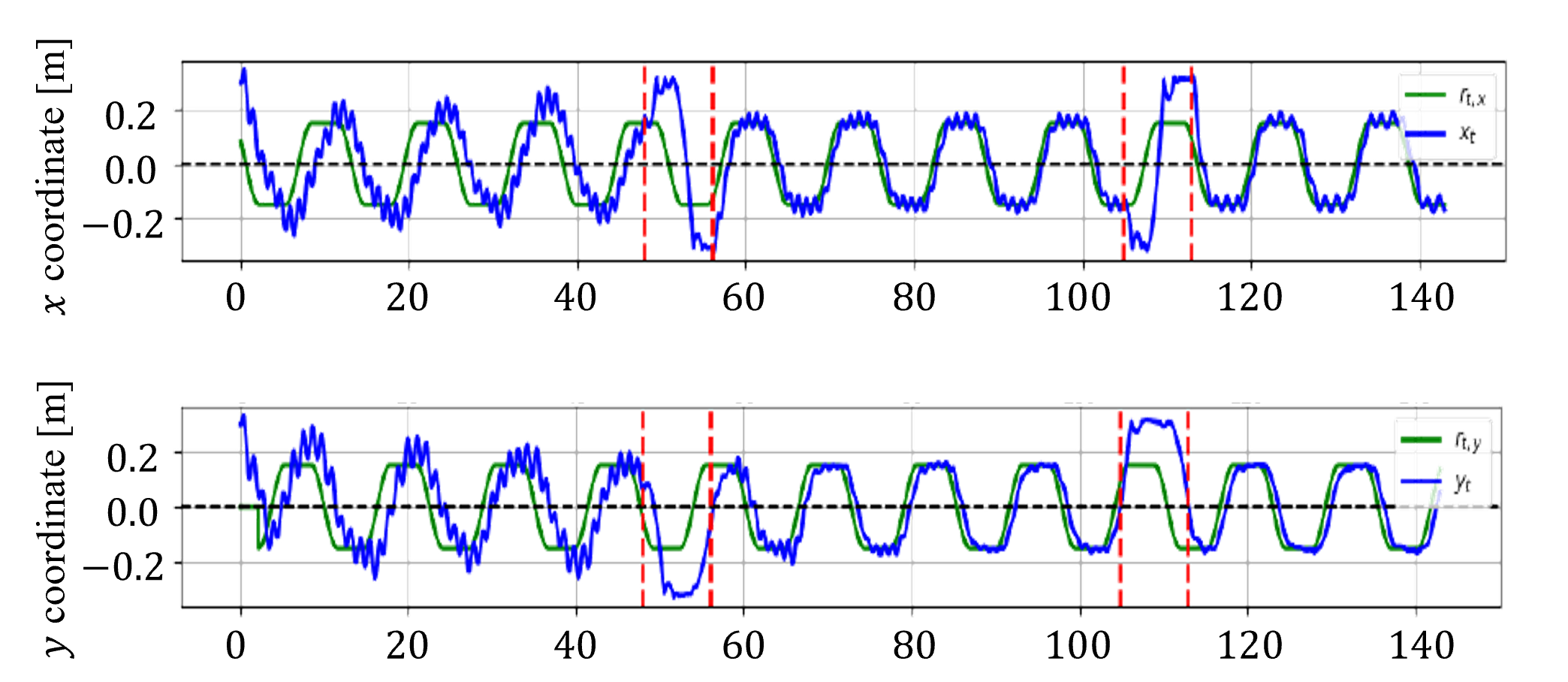}
		\caption{Experiment results on the real system after pre-training with Algorithm~\ref{alg:meta_training_RL}. The red dashed lines indicate the input intermezzo.}\label{fig:ball_imaml}
	\end{figure}
	
	\begin{figure}[!tb]
		\centering
		\includegraphics[width=\linewidth]{   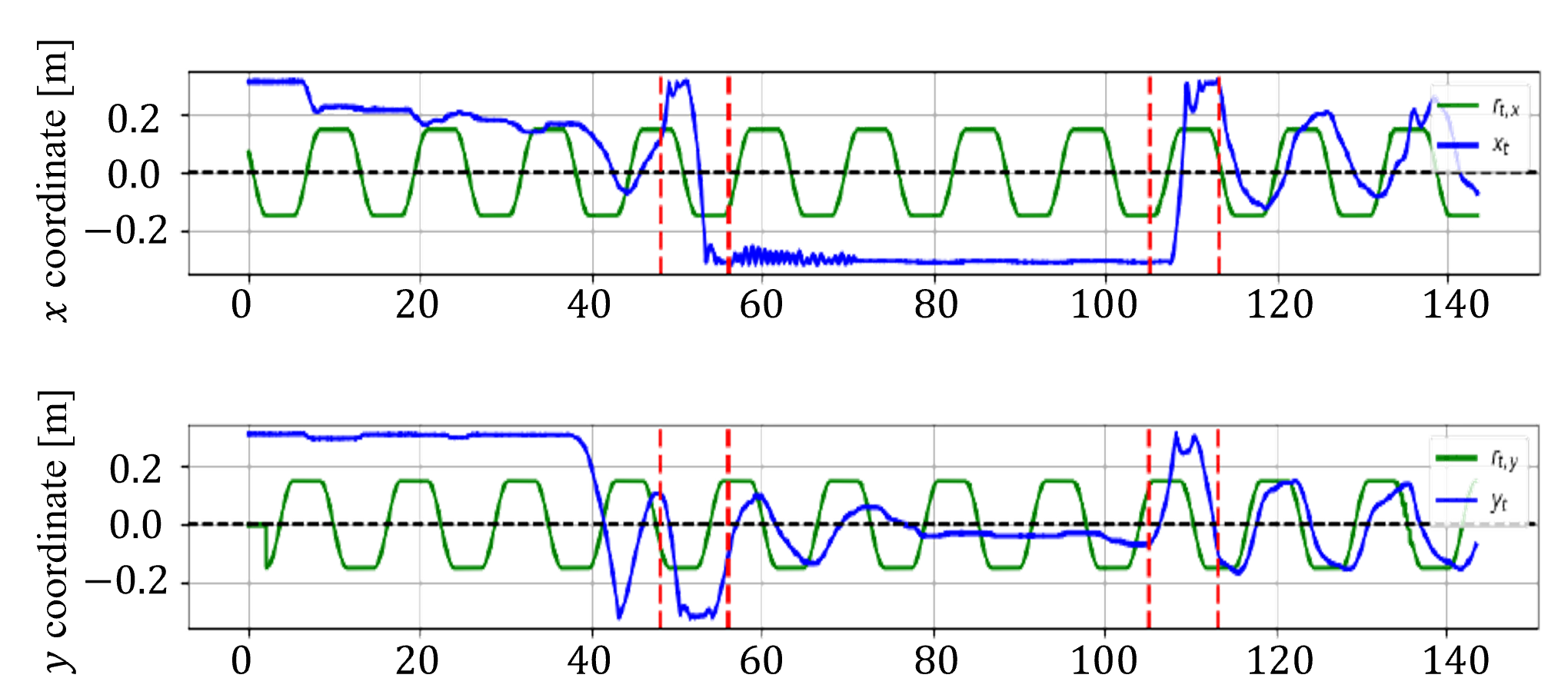}
		\caption{Experiment results of DQN by only using data from the real system.  The red dashed lines indicate the input intermezzo.}\label{fig:sl}
	\end{figure}
	
	%%%%%%%%%%%%%%%%%%%%%%%%%%%%%%%%%%%%%%%%%%%%%%%%%%%%%%%%%%%%%%%%%%%%%%%%%%%%%%%%%%%%%%%%
	\section{Conclusion}\label{sec:conclude}
	This paper proposes a meta-learning framework based on the iMAML algorithm for optimal control in uncertain nonlinear systems. By pretraining with offline source data and fine-tuning with (limited) online target data, the framework enables fast adaptation and improved control performance. We also introduce two algorithms that can integrated into our framework for controller design, utilizing MPC and deep Q-networks respectively. \revise{Theoretical analysis on convergence of the meta-learning framework and stability of the learning-based control algorithms is further given.} Numerical and hardware experiments finally confirm the superiority of our methods over baseline approaches. \revise{Future work will focus on improving the scalability of the proposed framework and extending its applicability to more complex, real-world robotic systems, such as legged robots and autonomous vehicles. Additionally, we will explore integration with state-of-the-art RL algorithms such as soft actor-critic and model-based planning methods to further improve sample efficiency and tracking performance.}

\appendices

{\color{black}
\section{Theoretical Additions}\label{app:theory}
In this section, we present additional theoretical results on the convergence of the proposed framework and the stability of the resulting control algorithms. To highlight the generality of the analysis, we adopt the notation used in the general framework and denote the loss function by $\widehat{\ell}(\cdot)$, which corresponds to $\ell_{\text{SSM}}(\cdot)$ and $\ell_{\text{DQN}}(\cdot)$ in Algorithms~\ref{alg:meta_training} and \ref{alg:meta_training_RL}, respectively.

\textbf{(1) Convergence of the meta-training phase}

In Section~\ref{sec:meta_learning}, we propose a meta-learning framework, within which 
	the inner problem \eqref{eqn:inner_imaml} and outer problem \eqref{eqn:outer_prob} are solved using iterative methods \eqref{eqn:inner_gd} and \eqref{eqn:outer}, respectively.
To show their convergence, we adopt the standard assumptions:
\begin{assumption}
\label{ass:loss}
For all $\mathcal{D}$ and $\psi$, the loss functions $\widehat{\ell}(\mathcal{D}; \psi)$ satisfy:
\begin{enumerate}
    \item[(i)] (\textbf{$L$-smoothness}) The function $\widehat{\ell}(\mathcal{D}; \psi)$ is $L$-smooth in $\psi$, i.e.,
    \[
    ||\nabla \widehat{\ell}(\mathcal{D}; \psi_1) - \nabla \widehat{\ell}(\mathcal{D}; \psi_2)||
    \le L ||\psi_1 - \psi_2||.
    \]

    \item[(ii)] (\textbf{Bounded Hessian}) The Hessian of $\widehat{\ell}(\mathcal{D}; \psi)$ is uniformly bounded, i.e.,
    \[
    ||\nabla^2 \widehat{\ell}(\mathcal{D}; \psi)|| \le H.
    \]
\end{enumerate}
\end{assumption}

\begin{assumption}
\label{ass:reg}
The regularization parameter satisfies $\gamma > H$.
\end{assumption}

Assumptions~\ref{ass:loss}--\ref{ass:reg} ensures the inner problem \eqref{eqn:inner_imaml} is $(\gamma-H)$-strongly convex, which guarantees: (i) linear convergence of gradient descent, as proved below, and (ii) invertibility of $Q^b = I + \frac{1}{\gamma}\nabla^2 \widehat{\ell} \succ 0$ in Lemma~\ref{lmm:implicit}, which is required for implicit differentiation.

The convergence of the meta-training phase is guaranteed by the following proposition:
\begin{proposition}[Meta-Training Convergence]
\label{prop:convergence}
Let Assumptions~\ref{ass:loss}--\ref{ass:reg} hold. 
With appropriate learning rate $\beta_{\mathrm{out}}$ and $M$ inner-loop gradient steps with step size $\beta_{\mathrm{in}} = 1/(L + \gamma)$, it follows that

%\[
%\todo{L(\omega_K)-L^* \leq \underbrace{\frac{C}{K}}_{\text{outer-loop rate}} + \underbrace{O\left(\left(1 - \frac{\gamma - H}{L + \gamma}\right)^M\right)}_{\text{inner-loop approximation error}},}
%\]

\[
\min_{k < K} ||\nabla L(\omega_k)||^2 \leq \underbrace{\frac{C}{K}}_{\text{outer-loop rate}} + \underbrace{O\left(\left(1 - \frac{\gamma - H}{L + \gamma}\right)^M\right)}_{\text{inner-loop approximation error}}
\]
where $C:=2L_{\mathrm{meta}}(L(\omega_0) - L^*)$ with the constant $L_{\mathrm{meta}}$ defined in the proof.
\end{proposition}

\begin{proof}
The meta-objective $L(\omega)$ is $L_{\mathrm{meta}}$-smooth by Assumption~\ref{ass:loss} and standard chain rule arguments; the explicit form depends on $L$, $H$, and $\gamma$ (see [Theorem 1, Rajeswaran2019]). Assumptions~\ref{ass:loss}--\ref{ass:reg} guarantee $L_{\mathrm{meta}} < \infty$. The $O(1/K)$ rate follows from gradient descent on smooth nonconvex functions. The constant $C = 2L_{\mathrm{meta}}(L(\omega_0) - L^*)$ arises from the descent lemma [Nesterov2013]: each gradient step decreases $L$ by at least $\frac{1}{2L_{\mathrm{meta}}}||\nabla L||^2$, and the total decrease is bounded by the initial suboptimality. 
The second term accounts for solving the inner problem approximately. Denoting the regularized inner problem (5) by $f$, its Hessian satisfies 
\[
\nabla^2 f = \nabla^2 \ell + \gamma I \succeq (\gamma - H)I,
\]
since eigenvalues of $\nabla^2 \ell$ lie in $[-H, H]$ by Assumption~\ref{ass:loss}(ii). Thus the inner problem is $(\gamma-H)$-strongly convex and $(L+\gamma)$-smooth. Strong convexity guarantees that each inner gradient step contracts the distance to the optimum by a fixed factor $(1 - \frac{\gamma-H}{L+\gamma}) < 1$, so after $M$ steps the error decays exponentially as $(1 - \frac{\gamma-H}{L+\gamma})^M$.
%The second term accounts for solving the inner problem approximately: since the inner problem is $(\gamma-H)$-strongly convex and $(L+\gamma)$-smooth, gradient descent converges at linear rate $(1 - \frac{\gamma-H}{L+\gamma})$.
\end{proof}

The bound is on $\min_{k<K} ||\nabla L(\omega_k)||^2$ instead of $L(\omega_K) - L^*$, because the meta-objective is nonconvex: gradient descent can only guarantee convergence to a stationary point (some iterate of the gradient descent where the gradient becomes small), not to the global minimum.

The step size $\beta_{\mathrm{in}} = 1/(L+\gamma)$ is the standard choice for gradient descent on smooth functions: since the regularized inner objective has smoothness constant $L + \gamma$, this step size guarantees convergence and achieves the linear rate stated above.

\textbf{(2) Bounded adaptation error in meta-inference phase}

By Proposition~\ref{prop:convergence}, after sufficient number of iterations, the meta-training phase converges to some $\omega_\infty$ which will then be fine-tuned on the target system:
	\begin{equation}\label{eqn:infer}
		\tilde{\omega}^*:=\argmin_{\psi} \; \widehat \ell(\mathcal{D}_{\text{target}}; \psi)+\frac{\gamma}{2}||\psi-\omega_\infty||^2.
	\end{equation}
	In practice, this is solved with a few gradient steps, using only limited data:
	\begin{equation}\label{eqn:infer_gd}
		\omega^*\gets \omega^* -\beta_{\text {in }} \big(\nabla_\psi\widehat\ell(\mathcal{D}_{\text{target}}; \psi)|_{\psi=\omega^*}+\gamma(\omega^*-\omega_\infty)).
	\end{equation}

We further denote the optimal parameters for the target system by $\omega^*_{\mathrm{opt}}$. The key assumption for adaptation is \emph{task similarity}: the meta-learned initialization $\omega_\infty$ should be close to $\omega^*_{\mathrm{opt}}$. Therefore, we assume some $\rho$ exists such that
\[
||\omega_\infty - \omega^*_{\mathrm{opt}}|| \leq \rho.
\]

We can prove that the adaptation error is bounded: 

\begin{proposition}[Adaptation Error]
\label{prop:adaptation}
Let Assumptions~\ref{ass:loss}--\ref{ass:reg} hold.
After $K_{\mathrm{adapt}}$ gradient steps of \eqref{eqn:infer_gd} with step size $\beta_{\mathrm{in}} = 1/(L+\gamma)$, the adapted parameters $\omega^*$ satisfy:
\begin{equation*}
\begin{aligned}
 &\ell(\mathcal{D}_{\mathrm{target}}; \omega^*) - \ell(\mathcal{D}_{\mathrm{target}}; \omega^*_{\mathrm{opt}}) \\&\leq \underbrace{\frac{\gamma \rho^2}{2}}_{\text{regularization bias}} + \underbrace{\frac{L+\gamma}{2} \left(1 - \frac{\gamma-H}{L+\gamma}\right)^{K_{\mathrm{adapt}}} ||\omega_\infty - \tilde{\omega}^*||^2}_{\text{optimization error}}.   
\end{aligned}    
\end{equation*}
\end{proposition}

The proof compares the adapted parameters $\omega^*$ to two reference points: (i) the exact minimizer $\tilde{\omega}^*$ of the regularized objective, and (ii) the unconstrained optimum $\omega^*_{\mathrm{opt}}$. The regularization bias arises from (ii), while the optimization error arises from (i).

\begin{proof}
Let $f(\psi) := \widehat\ell(\mathcal{D}_{\mathrm{target}}; \psi) + \frac{\gamma}{2}||\psi - \omega_\infty||^2$ denote the regularized objective.

\textbf{Step 1 (Optimality gap for regularized objective).} By the optimality of $\tilde{\omega}^*$:
\[
f(\tilde{\omega}^*) \leq f(\omega^*_{\mathrm{opt}}) = \ell(\mathcal{D}_{\mathrm{target}}; \omega^*_{\mathrm{opt}}) + \frac{\gamma}{2}||\omega^*_{\mathrm{opt}} - \omega_\infty||^2.
\]

\textbf{Step 2 (Optimization error).} The regularized objective $f$ is $(\gamma-H)$-strongly convex and $(L+\gamma)$-smooth. Gradient descent with step size $1/(L+\gamma)$ ensures ([Nesterov2013]):
\[
||\omega^* - \tilde{\omega}^*||^2 \leq \left(1 - \frac{\gamma-H}{L+\gamma}\right)^{K_{\mathrm{adapt}}} ||\omega_\infty - \tilde{\omega}^*||^2.
\]
By $(L+\gamma)$-smoothness of $f$:
\[
f(\omega^*) \leq f(\tilde{\omega}^*) + \frac{L+\gamma}{2}||\omega^* - \tilde{\omega}^*||^2.
\]

\textbf{Step 3 (Combining).} From Steps 1 and 2:
\[
f(\omega^*) \leq f(\omega^*_{\mathrm{opt}}) + \frac{L+\gamma}{2}||\omega^* - \tilde{\omega}^*||^2.
\]
Expanding $f$ and using $||\omega^*_{\mathrm{opt}} - \omega_\infty|| \leq \rho$:
\begin{equation*}
 \begin{split}
&\ell(\mathcal{D}_{\mathrm{target}}; \omega^*) + \frac{\gamma}{2}||\omega^* - \omega_\infty||^2 \\&\leq \ell(\mathcal{D}_{\mathrm{target}}; \omega^*_{\mathrm{opt}}) + \frac{\gamma \rho^2}{2} + \frac{L+\gamma}{2}||\omega^* - \tilde{\omega}^*||^2.
\end{split}   
\end{equation*}
Since $\frac{\gamma}{2}||\omega^* - \omega_\infty||^2 \geq 0$, we obtain the result.
\end{proof}

The bound in Proposition~\ref{prop:adaptation} has two terms:
\begin{itemize}
    \item \textbf{Regularization bias} $(\gamma \rho^2/2)$: The cost of pulling $\omega^*$ toward $\omega_\infty$. The bias is small when the source systems are similar to the target system ($\rho$ small).
    \item \textbf{Optimization error}: This decays exponentially with $K_{\mathrm{adapt}}$.
\end{itemize}
This formalizes the ``few-shot'' benefit: when $\rho$ is small, good performance is achievable with few adaptation steps.

\textbf{(3) Connection to control performance}

We finally base the stability results for MPC (i.e. Algorithm~\ref{alg:meta_training}) on the connection of meta-learning to robust MPC theory.

\begin{remark}[Robust MPC Stability]~\label{rmk:robustMPC}
When MPC uses a model with bounded model mismatch, the closed-loop system achieves input-to-state stability (ISS) under standard assumptions (stabilizability, detectability, appropriate terminal cost). The tracking error converges to a neighborhood of zero, with size proportional to the model error. See [Limon2009] for precise statements.
\end{remark}

Proposition~\ref{prop:adaptation} shows that effective meta-learning (small $\rho$) yields smaller loss after adaptation. In Algorithm~1, this translates to smaller model mismatch in MPC. Combined with robust MPC theory in Remark~\ref{rmk:robustMPC}, a clear pathway is built: effective meta-learning reduces model error, which in turn guarantees the input-to-state stability of system and leads to tighter tracking performance. This relationship is consistent with our experimental results, where iMAML achieves lower prediction error and better tracking performance than the baselines.

\begin{remark}[Limitations and Future Work]
\begin{itemize}
    \item The learned dynamics $(A_z, B_z, C_z)$ must satisfy stabilizability and detectability for MPC stability. We observe this empirically; enforcing it during training is future work.
    %\item The encoder recomputation (Algorithm~1 recomputes $z_t$ at each step) introduces additional model mismatch not analyzed here.
    \item For DQN, rigorous guarantees under neural function approximation remain an open problem, as discussed in Remark~\ref{rmk:dqn}.
\end{itemize}
\end{remark}}
    
	\bibliographystyle{IEEEtran}
	\bibliography{reference} 
\end{document}